\documentclass{article}

\usepackage{pdfsync}
\usepackage{graphicx} 

\usepackage[colorlinks,citecolor=blue,urlcolor=blue]{hyperref}
\usepackage[round,authoryear]{natbib}
\usepackage{algorithm}
\usepackage{algorithmic}

\usepackage{amsthm}
\usepackage{mathtools}
\usepackage{amssymb}
\usepackage{amsmath} 
\usepackage{xcolor}
\usepackage{graphicx}
\usepackage{epstopdf}
\usepackage{algorithm,algorithmic}
\usepackage{verbatim} 
\usepackage[capitalise]{cleveref}
\usepackage{nicefrac}
\usepackage[margin=1.2in]{geometry}
\allowdisplaybreaks

\usepackage[normalem]{ulem}

\newtheorem{theorem}{Theorem}

\newtheorem{lem}{Theorem}
\newtheorem{assump}{Theorem}

\newtheorem{lemma}[lem]{Lemma}

\newtheorem*{theorem*}{Theorem}
\newtheorem*{lemma*}{Lemma}

\newtheorem{defn}{Definition}
\newtheorem{assumption}[assump]{Assumption}

\newcommand{\Kh}{\hat K}
\newcommand{\calN}{\mathcal{N}}
\newcommand\independent{\protect\mathpalette{\protect\independenT}{\perp}}
\def\independenT#1#2{\mathrel{\rlap{$#1#2$}\mkern2mu{#1#2}}}

\DeclareMathOperator*{\Tr}{\mathbf{Tr}}
\newcommand{\Trr}[1]{\mathbf{Tr}\left(#1\right)}
\newcommand{\as}{\text{ a.s.}}

\newcommand{\Ah}{\hat A}
\newcommand{\Bh}{\hat B}

\newcounter{assumption}
\renewcommand{\theassumption}{A\arabic{assumption}}

\makeatletter
\newcommand*{\rom}[1]{\expandafter\@slowromancap\romannumeral #1@}
\makeatother

\newcommand{\norm}[1]{\left\Vert#1\right\Vert}

\newcommand{\maxEig}[1]{\lambda_{\max}\left(#1\right)}

\newcommand{\calF}{\mathcal{F}}
\newcommand{\calS}{\mathcal{S}}
\newcommand{\N}{\natural}

\newcommand{\Gamlow}{\underline{\Gamma}}
\newcommand{\Gambar}{\bar{\Gamma}}

\newcommand{\lnorm}[1]{\norm{#1}}

\newcommand{\mathset}[1]{\left\{#1\right\}}

\newcommand{\inputdim}{n}

\renewcommand{\natural}{\mathbb N}                   
\newcommand{\real}{\mathbb R}                        
\newcommand{\R}{{\mathbb{R}}}                        

\renewcommand{\P}{{\mathbb P}}                         
\newcommand{\EE}[1]{{\mathbb E}\left(#1\right)}      
\newcommand{\E}{{\mathbb E}}                         
\newcommand{\Var}[1]{{\mathrm{Var}}\left(#1\right)}  

\newcommand{\argmin}{\mathop{\rm argmin}}

\newcommand{\beq}{\begin{equation}}
\newcommand{\eeq}{\end{equation}}
\newcommand{\beqa}{\begin{eqnarray}}
\newcommand{\eeqa}{\end{eqnarray}}
\newcommand{\beqan}{\begin{eqnarray*}}
\newcommand{\eeqan}{\end{eqnarray*}}
\newcommand{\ben}{\begin{eqnarray*}}
\newcommand{\een}{\end{eqnarray*}}
\newcommand{\bea}{\begin{align*}}
\newcommand{\eea}{\end{align*}}

\newcommand{\sfrac}{\nicefrac}

\renewcommand{\phi}{\varphi}

\newcommand{\PP}[1]{\mathbb{P}\left( #1\right)}

\newcommand{\iidsim}{\overset{i.i.d.}{\sim}}

\newcommand{\bigOmega}[1]{\Omega\left(#1\right)}

\newcommand{\bigO}[1]{O\left(#1\right)}

\newcommand{\bigOp}[1]{O_p\left(#1\right)}

\newcommand{\logOO}{\tilde O}

\newcommand{\logg}[1]{\log\left(#1\right)}

 \newcommand{\keyboundappliedp}[1]{
 R \sqrt{2 \log\left( \frac{\det (V_{#1})^{\sfrac12}\det(\lambda I)^{\sfrac{\kern-2pt-\kern-2pt 1}{2}}}{\delta} \right)}
}

\newcommand{\ebrace}[1]{\left\lbrace #1 \right\rbrace}

\newcommand{\expebrace}[1]{\exp \ebrace{#1}}

\DeclarePairedDelimiter{\floor}{\lfloor}{\rfloor}

\crefname{thm}{Theorem}{Theorems}
\crefname{lemma}{Lemma}{Lemmas}
\crefname{Lemma}{Lemma}{Lemmas}
\crefname{prop}{Proposition}{Propositions}
\crefname{subsubsubappendix}{Appendix}{Appendices}

\title{
Rate-matching the regret lower-bound in the\\linear quadratic regulator with unknown dynamics
}

\author{
Feicheng Wang and Lucas Janson
}

\date{Department of Statistics, Harvard University}

\begin{document}

\maketitle

\begin{abstract}
The theory of reinforcement learning currently suffers from a mismatch between its empirical performance and the theoretical characterization of its performance, with consequences for, e.g., the understanding of sample efficiency, safety, and robustness. The linear quadratic regulator with unknown dynamics is a fundamental reinforcement learning setting with significant structure in its dynamics and cost function, yet even in this setting there is a gap 
between the best known regret lower-bound of $\Omega_p(\sqrt{T})$ and the best known upper-bound of $O_p(\sqrt{T}\,\text{polylog}(T))$. The contribution of this paper is to close that gap by establishing a novel regret upper-bound of $O_p(\sqrt{T})$. 
Our proof is constructive in that it analyzes the regret of a concrete algorithm, and simultaneously establishes an estimation error bound on the dynamics of $O_p(T^{-1/4})$ which is also the first to match the rate of a known lower-bound. The two keys to our improved proof technique are (1) a more precise upper- and lower-bound on the system Gram matrix and (2) a self-bounding argument for the expected estimation error of the optimal controller.
\end{abstract}

\section{Introduction}
We have witnessed great progress in reinforcement learning (RL) beating human professionals in various challenging games like GO \citep{silver2016mastering}, Starcraft \rom{2} \citep{vinyals2019grandmaster} and Dota 2 \citep{berner2019dota}. Successes in these highly complex simulation environments have led to an increasing drive to apply RL in real world data-driven systems such as self driving cars \citep{kiran2021deep} and automatic robots \citep{levine2016end}. Yet real-world deployment comes with increased risks and costs, and as such has been hindered by the field's limited understanding of the gap between theoretical bounds and the empirical performance of RL.
One line of attack for this problem is to deepen our understanding of relatively simple yet fundamental systems such as the linear quadratic regulator (LQR) with unknown dynamics.

\subsection{Problem statement}
In the LQR problem, the system obeys the following dynamics starting from $t=0$:
\begin{equation}
\label{eq:system eq}
    x_{t+1} = Ax_t + Bu_t + \varepsilon_t,
\end{equation}
where $x_t \in \real^{n}$ represents the state of the system at time $t$ and starts at some initial state $x_0$, $u_t \in \real^d$ represents the action or control applied at time $t$, $\varepsilon_t \iidsim \calN(0, \sigma_\varepsilon ^2 I_n)$
is the system noise, and $A\in\real^{n\times n}$ and $B\in\real^{n\times d}$ are  matrices determining the system's linear dynamics. 
The goal is to find an algorithm $U$ that, at each time $t$, outputs a control $u_t = U(H_t)$ that is computed using the entire thus-far-observed history of the system $H_t = \{x_t, u_{t-1}, x_{t-1},\dots,u_1,x_1,u_0\}$ to maximize the system's function while minimizing control effort.
The cost of the LQR problem up to a given finite time $T$ is quadratic:
\begin{equation}
\label{eq:cost definition}
    \mathcal{J}(U,T) = 
    \sum_{t=1}^{T} \left(x_t^\top Qx_t + u_t^\top Ru_t\right)
\end{equation}
for some known positive definite matrices $Q\in\mathbb{R}^{n\times n}$ and $R\in\mathbb{R}^{d\times d}$. When the system dynamics $A$ and $B$ are also known and $T\rightarrow\infty$, the cost-minimizing algorithm is known: $u_t^* = U^*(H_t) = Kx_t$, where $K\in\real^{d\times n}$ is the efficiently-computable solution to a system of equations that only depend on $A$, $B$, $Q$, and $R$. Like the Gaussian linear model in supervised learning, the aforementioned linear-quadratic problem is foundational to control theory because it is conceptually simple yet it provides a remarkably good description for some real-world systems. In fact, many systems are close to linear over their normal range of operation, and linearity is an important factor in system design \citep{recht2019tour}.

In this paper we consider the case when the system dynamics $A$ and $B$ are unknown.
Intuitively, one might hope that after enough time observing a system controlled by almost any algorithm, one should be able to estimate $A$ and $B$ (and hence $K$) fairly well and thus be able to apply an algorithm quite close to $U^*$. Indeed the key challenge in LQR with unknown dynamics, as in any reinforcement learning problem, is to trade off \emph{exploration} (actions that help estimate $A$ and $B$) with \emph{exploitation} (actions that minimize cost). We will quantify the cost of an algorithm by its \emph{regret}, which is the difference in cost achieved by the algorithm and that achieved by the oracle optimal controller $U^*$:
\begin{equation*}
    \mathcal{R}(U,T) = \mathcal{J}(U,T)-\mathcal{J}(U^*,T).
\end{equation*}

The best known upper-bound for the regret of LQR with unknown dynamics is $O_p(\sqrt{T}\,\text{polylog}(T))$, which contains a polylogarithmic factor of $T$ that is not present in the best known lower-bound of $\Omega_p(\sqrt{T})$. This paper closes that rate gap by establishing a novel regret upper-bound of $O_p(\sqrt{T})$, where the improvement comes from a more careful bound of the system Gram matrix combined with a self-bounding argument for the expected estimation error. As part of our proof, we show that the algorithm that achieves our optimal rate of regret also produces data that can be used for system identification (estimation of $A$ and $B$) at a rate of $\norm{\Ah - A}_2 = \norm{\Bh - B}_2 = O_p(T^{-1/4})$, which is also tighter than the best known bounds of $O_p(T^{-1/4}\,\text{polylog}(T))$ for data from an algorithm achieving $O_p(\sqrt{T}\,\text{polylog}(T))$ regret, where the tildes hide polylogarthmic terms in $T$.

\subsection{Related works}
Many works have studied optimal rates of regret in RL. In bandits, matching upper- and lower-bounds have been found as $\Theta_p(\logg{T})$ for the distribution-dependent regret \citep{lai1985asymptotically,auer2002finite,magureanu2014lipschitz,agrawal2013further,komiyama2015regret,garivier2018kl} and $\Theta_p(\sqrt{T})$ for the distribution-free regret \citep{agrawal2013further,osband2016lower,garivier2018kl,li2019nearly,hajiesmaili2020adversarial}.

For Markov decision processes (MDPs), most work has considered finite state and action spaces. In this setting, a matching upper- and lower-bound of $\Theta_p(\logg{T})$ is known for the distribution-dependent regret \citep{burnetas1997optimal,tewari2007optimistic,ok2018exploration,tirinzoni2021fully,xu2021fine}, while the best known upper-bound of $O_p(\sqrt{T}\,\text{polylog}(T))$ for the distribution-free regret \citet{jaksch2010near,azar2017minimax,agrawal2017posterior,simchowitz2019non,xiong2021randomized} has a polylogarithmic gap with the best-known lower-bound of $\Omega_p(\sqrt{T})$ \citep{jaksch2010near,osband2016lower,azar2017minimax}. 


The LQR system is an MDP with \emph{continuous} state and action spaces, and has received increasing interest recently. For the LQR system with unknown dynamics,
\citet{simchowitz2020naive} proved a $\Omega_p(\sqrt{T})$ lower-bound for the regret along with an upper-bound of $O_p(\sqrt{T\log(\frac{1}{\delta}}))$ with probability $1-\delta$ under the condition $\delta<1/T$, so that the upper-bound contains an implicit additional $\log^{1/2}(T)$ term. Other $O_p(\sqrt{T}\,\text{polylog}(T))$ regret upper-bounds for LQR with uknown dynamics have been established elsewhere \citep{faradonbeh2018input,faradonbeh2018optimality,mania2019certainty,abbasi2011regret,ibrahimi2012efficient,faradonbeh2017finite,cohen2019learning,ouyang2017learning,faradonbeh2018optimality,abeille2018improved,wang2020exact}, but to the best of our knowledge, no existing work has matched the $\Omega_p(\sqrt{T})$ lower-bound until the present paper.
Our proof borrows many insightful results and ideas from a number of these prior works, especially \citet{simchowitz2018learning,fazel2018global,simchowitz2020naive,wang2020exact}.

\subsection{Algorithm and assumptions}

Throughout the paper, we make only one assumption on the true system parameters:
\begin{assumption}[Stability]
\label{asm:InitialStableCondition}
    Assume the system is \emph{stabilizable}, i.e., there exists $K_0$ such that the spectral radius (maximum absolute eigenvalue) of $A+BK_0$ is strictly less than 1.
\end{assumption}
\noindent Under Assumption \ref{asm:InitialStableCondition}, it is well known that 
there is a unique optimal controller $u_t = K x_t$ \citep{arnold1984generalized} which can be computed from $A$ and $B$, where 
\begin{equation}
\label{eq:ControllerK}
K = - (R + B^\top P B)^{-1}B^\top P A
\end{equation}
and $P$ is the unique positive definite solution to the discrete algebraic Riccati equation (DARE): 
\begin{align}
\label{eq:riccati}
  P = A^\top P A - A^\top P B (R + B^\top P B)^{-1} B^\top P A + Q.
\end{align}

In this paper we will consider the same algorithm as in \citet{wang2020exact}, reproduced here as \cref{alg:myAlg}, which is a \emph{noisy certainty equivalent control} algorithm. In particular, at every round $t$, we generate an estimate $\Kh_t$ for $K$, and then apply control $u_t = \Kh_t x_t + \eta_t$ as a substitute of the optimal unknown control $u_t = Kx_t$, where $\eta_t \sim \calN(0, t^{-1/2}I_d)$ is a noise term whose variance shrinks at a carefully chosen rate in $t$ so as to rate-optimally trade off exploration and exploitation. Note that \cref{alg:myAlg} is step-wise and online, i.e., it does not rely on independent restarts or episodes of any kind and does not depend on the time horizon $T$. The two things it does rely on, which are standard in the literature (see, e.g., \citet{dean2018regret}), are the knowledge of a stabilizing controller $K_0$ and an upper-bound $C_K$ on the spectral norm of the optimal controller $K$; $C_x$ and $\sigma_\eta$ are also inputs but can take any positive numbers. 


    \begin{algorithm}[ht]
    \caption{Stepwise Noisy Certainty Equivalent Control}
    \begin{algorithmic}[1]
      \REQUIRE{Initial state $x_0$, stabilizing control matrix $K_0$, 
        scalars $C_{x} > 0$, $C_{K} > \norm{K}$, $\sigma_\eta > 0$.
        }
        \STATE Let $u_0 = K_0x_0 + \eta_0$ and $u_1 = K_0x_1 + \eta_1$, with $\eta_0,\eta_1\stackrel{iid}{\sim}\calN(0, \sigma_\eta^2 I_d)$.
        \FOR{$t = 2,3,\dots$}
            \STATE Compute
            \begin{equation}
                \label{eq: AhBh estimator}
                (\Ah_{t-1}, \Bh_{t-1}) \in \argmin_{(A', B')} \sum_{k=0}^{t-2} \lnorm{x_{k+1} - A' x_k - B' u_k}^2
            \end{equation}
            and if stabilizable, plug them into the DARE (\cref{eq:ControllerK,eq:riccati}) to compute $\Kh_t$, otherwise set $\Kh_t=K_0$.
            \label{line:ols}
            \STATE If 
            $\norm{x_{t}} \gtrsim C_x\log(t)$ or $\norm{\Kh_t} \gtrsim C_K$, reset $\Kh_t = K_0$.\label{line:check}
            \STATE Let
            \begin{equation}\label{eq:Myinput}
            u_t = \Kh_tx_t + \eta_t,\hspace{1cm} \eta_t \stackrel{iid}{\sim} \calN(0, \sigma_\eta^2 t^{-1/2}I_\inputdim)
            \end{equation}\label{line:controller}
        \ENDFOR
    \end{algorithmic}
    \label{alg:myAlg}
    \end{algorithm}

\subsection{Notation}
Throughout our proofs, we use $X \lesssim Y$ (resp. $X\gtrsim Y$) as shorthand for the inequality $X \le CY$ (resp. $X\ge CY$) for some constant $C$. $X \eqsim Y$ means both $X \lesssim Y$ and $X\gtrsim Y$. We will almost always establish such relations between quantities that (at least may) depend on $T$ and show that they hold with at least some stated probability $1-\delta$; in such cases, we will always make all dependence on both $T$ and $\delta$ explicit, i.e., the hidden constant(s) $C$ will never depend on $T$ or $\delta$, though they may depend on any other parameters of the system or algorithm, including $A$, $B$, $Q$, $R$, $\sigma_\epsilon^2$, $\sigma_\eta^2$, $K_0$, $C_x$, $C_K$.

\subsection{Outline}
In the remainder of this paper, we will present an outline of the proof of our improved regret upper-bound in two parts. First, in Section~\ref{sec: proof_est}, we will establish a novel $\bigOp{T^{-1/4}}$ bound on the estimation error of $\Ah_t$, $\Bh_t$, and $\Kh_t$ from \cref{alg:myAlg}. Then, in Section~\ref{sec: proof_reg}, we will leverage this tighter estimation error bound to establish our $\bigOp{\sqrt{T}}$ bound on the regret of \cref{alg:myAlg}.

\section{Bounding the estimation error by $\bigOp{T^{-1/4}}$}
\label{sec: proof_est}

	    

Our bound on the estimation error starts with a key result from
\cite{simchowitz2018learning}, which relates the estimation error to the system Gram matrix via a lower- and upper-bound for it. 
The rest of the proof is primarily comprised of two parts. In the first part, we prove a more precise upper- and lower-bound on the system Gram matrix so that the two bounds are almost of the same order, which is crucial in removing the $\text{polylog}(T)$ in the estimation error bound. In the second part, we take the estimation error bound from plugging in the Gram matrix bounds from the first part and transform it into a self-bounding argument for the expected estimation error of the estimated dynamics that yields the $O_p(T^{-1/4})$ final rate for the estimation error.

To streamline notation, define $z_t = \begin{bmatrix}
x_t \\
u_t
\end{bmatrix}$ and $\Theta = [A, B]$, and correspondingly define $\hat{\Theta}_t = [\Ah_t,\Bh_t]$.
Then by Theorem 2.4 of \citet{simchowitz2018learning}, given a fixed $\delta \in (0,1)$, $T \in \N$ and $0 \preceq \Gamlow \preceq \Gambar \in \R^{(n+d) \times (n+d)}$ such that $\PP{\sum_{t=0}^{T-1} z_tz_t^\top \succeq T\Gamlow} \ge 1- \delta$ and $\P[\sum_{t = 0}^{T-1}  z_t z_t^\top \preceq T\Gambar] \ge 1-\delta$, when
\begin{equation}
\label{eq: T condition simchowitz}
    T \gtrsim  \log \left(\frac{1}{\delta}\right) + 1 +  \log \det (\Gambar \Gamlow^{-1}),
\end{equation}
$\hat{\Theta}_T$
satisfies:
\begin{align}
\label{eq: simchowitz conclusion}
\P\left[\lnorm{\hat\Theta_T-\Theta} \gtrsim \sqrt{\frac{1 + \log \det \Gambar \Gamlow^{-1} + \log\left(\frac{1}{\delta}\right)}{T\lambda_{\min}(\Gamlow)}} \right] \le  \delta.
\end{align}

Similar upper-bounds to those that already exist in the literature (which contain extra polylog$(T)$ terms compared to the best know lower-bound) can be achieved by taking $\Gamlow \eqsim T^{-1/2}I_{n+d}$ and $\Gambar \eqsim \log^2(T) I_{n+d}$, and we restate this result here (and prove it in Appendix~\ref{subsection: proof of lemma: delta theta bound}) for completeness.
\begin{lemma}[Estimation error bound with polylog$(T)$ term]
\label{lemma: delta theta bound}
\cref{alg:myAlg} applied to a system described by \cref{eq:system eq} under Assumption \ref{asm:InitialStableCondition} satisfies, when $0 < \delta < 1/2$, for any $T \gtrsim \log(1/\delta)$,
\begin{align}
\label{eq: naive theta bound with log}
\P\left[\lnorm{\hat\Theta_T-\Theta} \gtrsim T^{-1/4}\sqrt{
\left(
\log T + \log\left(\frac{1}{\delta}\right)
\right)} 
\right] \le  \delta.
\end{align}
%
\end{lemma}

In order to improve this $\bigOp{T^{-1/4}\log^{1/2}(T)}$ bound to the desired $\bigOp{T^{-1/4}}$, we need tighter lower- and upper-bounds $\Gamlow$ and $\Gambar$ for $\sum_{t=0}^{T-1} z_t z_t^\top$. The following Lemma is one of the key steps in our proof.
\begin{lemma}
\label{lemma: better upper and lower bounds of GT}
\cref{alg:myAlg} applied to a system described by \cref{eq:system eq} under Assumption \ref{asm:InitialStableCondition} satisfies, for any $0 < \delta < 1/2$ and $T \gtrsim \log^3(1/\delta)$, with probability at least $1-\delta$: 
\begin{align}
\label{eq: GT upper and lower bounds}
\begin{split}
    &
    T\Gamlow :=
    \begin{bmatrix}
            I_n \\
            K
        \end{bmatrix} T
        \begin{bmatrix}
            I_n \\
            K
        \end{bmatrix}^\top + 
        \begin{bmatrix}
-K^\top \\
I_d
\end{bmatrix} T^{1/2}
\begin{bmatrix}
-K^\top \\
I_d
\end{bmatrix}^\top \\
& \qquad \precsim 
\sum_{t=0}^{T-1} 
z_t z_t^\top 
\precsim \left(
    \frac1\delta 
    \begin{bmatrix}
I_n \\
K
\end{bmatrix} T
\begin{bmatrix}
I_n \\
K
\end{bmatrix}^\top + 
\begin{bmatrix}
-K^\top \\
I_d
\end{bmatrix} \maxEig{ \sum_{t=0}^{T-1}\Delta_t\Delta_t^\top}
\begin{bmatrix}
-K^\top \\
I_d
\end{bmatrix}^\top 
\right)
:= T\Gambar,
\end{split}
\end{align}
where $\Delta_t := (\Kh_t - K)x_t + \eta_t$.
\end{lemma}
The complete proof of Lemma \ref{lemma: better upper and lower bounds of GT} can be found at \cref{Proof of lemma: better upper and lower bounds of GT}. 
\begin{proof}(sketch)
$G_T := \sum_{t=0}^{T-1} 
z_t z_t^\top$ can be represented as a summation of two parts:
\begin{equation}
\label{eq: end of equation}
G_T = \sum_{t=0}^{T-1} 
z_t z_t^\top = 
\begin{bmatrix}
I \\
K
\end{bmatrix}
\sum_{t=0}^{T-1} 
x_t x_t^\top
\begin{bmatrix}
I \\
K
\end{bmatrix}^\top +
    \sum_{t=0}^{T-1}
    \begin{bmatrix}
    0_n &  x_t \Delta_t^\top\\
    \Delta_tx_t^\top  & \Delta_t\Delta_t^\top + K x_t \Delta_t^\top + \Delta_tx_t^\top K^\top  \\
    \end{bmatrix}.
\end{equation}
We consider the dominating part $\begin{bmatrix}
I \\
K
\end{bmatrix}
\sum_{t=0}^{T-1} 
x_t x_t^\top
\begin{bmatrix}
I \\
K
\end{bmatrix}^\top $
(smallest eigenvalue scales with $T$) and the remainder part 
$\sum_{t=0}^{T-1}
    \begin{bmatrix}
    0_n &  x_t \Delta_t^\top\\
    \Delta_tx_t^\top  & \Delta_t\Delta_t^\top + K x_t \Delta_t^\top + \Delta_tx_t^\top K^\top  \\
    \end{bmatrix}
$ separately. 
We then prove in Lemma \ref{lemma: large eigenvalue part} that with probability at least $1-\delta$:
\begin{equation}
\label{eq: partial bound of GT}
\begin{bmatrix}
I \\
K
\end{bmatrix} T 
\begin{bmatrix}
I \\
K
\end{bmatrix}^\top
\preceq
\begin{bmatrix}
I \\
K
\end{bmatrix}
\sum_{t=0}^{T-1} 
x_t x_t^\top
\begin{bmatrix}
I \\
K
\end{bmatrix}^\top \preceq 
1/\delta
\begin{bmatrix}
I \\
K
\end{bmatrix} T 
\begin{bmatrix}
I \\
K
\end{bmatrix}^\top
\end{equation} 
These bounds reflect the intuition that $x_t$ should converge to a stationary distribution, making each of the summands $x_tx_t^\top$ of constant order.

\paragraph{Lower bound}
\cref{eq: partial bound of GT} provides a partial lower bound for $G_T$: with probability at least $1-\delta$,
\begin{equation}
\label{eq: partial lower bound of GT}
G_T 
\succeq%
\begin{bmatrix}
I \\
K
\end{bmatrix}
\sum_{t=0}^{T-1} 
x_t x_t^\top
\begin{bmatrix}
I \\
K
\end{bmatrix}^\top
\succsim%
\begin{bmatrix}
I \\
K
\end{bmatrix} T 
\begin{bmatrix}
I \\
K
\end{bmatrix}^\top.
\end{equation} 
This part only covers the subspace spanned by $\begin{bmatrix}
I \\
K
\end{bmatrix}$; we still need to consider a general bound for the whole matrix $G_T$. Lemma 34 of \citet{wang2020exact} gives a high probability lower-bound $G_T \succsim T^{1/2}I_{n+d}$.
Combining this and \cref{eq: partial lower bound of GT}, with high probability:
\begin{equation*}
    G_T + G_T \succsim 
        \begin{bmatrix}
            I_n \\
            K
        \end{bmatrix} T
        \begin{bmatrix}
            I_n \\
            K
        \end{bmatrix}^\top + 
        T^{1/2} I_{n+d}
    \succsim 
    \begin{bmatrix}
            I_n \\
            K
        \end{bmatrix} T
        \begin{bmatrix}
            I_n \\
            K
        \end{bmatrix}^\top + 
        \begin{bmatrix}
-K^\top \\
I_d
\end{bmatrix} T^{1/2}
\begin{bmatrix}
-K^\top \\
I_d
\end{bmatrix}^\top.
\end{equation*}
\paragraph{Upper bound}
The argument for our upper-bound divides $\real^{n+d}$ into two orthogonal subspaces spanned by the columns of
$\begin{bmatrix}
I_n \\
K
\end{bmatrix}$ and $\begin{bmatrix}
-K^\top \\
I_d
\end{bmatrix}$, and essentially bounds $\xi^\top G_T\xi$ separately by order $T$ and $\maxEig{ \sum_{t=0}^{T-1}\Delta_t\Delta_t^\top}$ for the two subspaces, respectively. In particular, for any $\xi_1$ in the span of $\begin{bmatrix}
I_n \\
K
\end{bmatrix}$ and $\xi_2$ in the span of $\begin{bmatrix}
-K^\top \\
I_d
\end{bmatrix}$,
\begin{align*}
    (\xi_1 + \xi_2)^\top G_T (\xi_1 + \xi_2)
    &\le 2 \xi_1^\top G_T \xi_1 + 2 \xi_2^\top G_T \xi_2 \\
    &\text{(using \cref{eq: end of equation}, because $\xi_2$ is orthogonal to $\begin{bmatrix}
        I \\
        K
        \end{bmatrix}
        \sum_{t=0}^{T-1} 
        x_t x_t^\top
        \begin{bmatrix}
        I \\
        K
        \end{bmatrix}^\top$)} \\
    & \lesssim 2 \xi_1^\top G_T \xi_1 + 
        2\norm{\xi_2}^2
        \maxEig{ \sum_{t=0}^{T-1}\Delta_t\Delta_t^\top}
    \\
     & \text{(we show in \cref{subsection: Proof of lemma: upper bound of $G_T$} that $ \PP{G_T \precsim \frac{1}{\delta} TI_{n+d} } \ge 1-\delta.$)} \\
    & \lesssim \frac1\delta 
    T \norm{\xi_1}^2 + 
     \norm{\xi_2}^2
    \maxEig{ \sum_{t=0}^{T-1}\Delta_t\Delta_t^\top},
\end{align*}
where the last inequality holds with high probability. This last expression can in turn be bounded by 
\begin{align*}
    (\xi_1 + \xi_2)^\top
    \left(
    \frac1\delta 
    \begin{bmatrix}
I_n \\
K
\end{bmatrix} T
\begin{bmatrix}
I_n \\
K
\end{bmatrix}^\top + 
\begin{bmatrix}
-K^\top \\
I_d
\end{bmatrix} \maxEig{ \sum_{t=0}^{T-1}\Delta_t\Delta_t^\top}
\begin{bmatrix}
-K^\top \\
I_d
\end{bmatrix}^\top
\right)
(\xi_1 + \xi_2),
\end{align*}
establishing the upper-bound from \cref{eq: GT upper and lower bounds}.
\end{proof}

In Lemma \ref{lemma: better upper and lower bounds of GT}, the upper bound $\Gambar$ and lower bound $\Gamlow$ have similar forms. 
Plugging them into \cref{eq: simchowitz conclusion} gives
that when $T \gtrsim \log^3(1/\delta)$, 
\begin{align}
\label{eq: theta bound with lambda max deltat}
\P\left[\lnorm{\hat\Theta_T-\Theta} \gtrsim 
\sqrt{\frac{1 + \log\left(\maxEig{ \sum_{t=0}^{T-1}\Delta_t\Delta_t^\top} T^{-1/2} \right) + \log\left(\frac{1}{\delta}\right)}{T^{1/2}}} \right] \le  \delta.
\end{align}

The following Lemmas \ref{lemma: lambda max sum of delta deltaT bound} and \ref{lemma: dtheta recursive} will connect the key term $\maxEig{ \sum_{t=0}^{T-1}\Delta_t\Delta_t^\top}$ in the estimation error bound of \cref{eq: theta bound with lambda max deltat} with the estimation error itself, setting up the self-bounding argument that is key to our main estimation error bound in \cref{theorem: theta bound without log}.

\begin{lemma}
\label{lemma: lambda max sum of delta deltaT bound}
\cref{alg:myAlg} applied to a system described by \cref{eq:system eq} under Assumption \ref{asm:InitialStableCondition} satisfies, for any $0 < \delta < 1/2$, $T \gtrsim \log^2(1/\delta)$,

\begin{equation*}
    \PP{ \maxEig{\sum_{t=0}^{T-1}\Delta_t\Delta_t^\top } \gtrsim
    1/\delta 
    \left(
    \sum_{t=1}^{T-1}\EE{t^{1/2}\norm{\Kh_t - K}^4 } +\log^2(1/\delta) +  T^{1/2} 
    \right)} \le 2\delta.
\end{equation*}
\end{lemma}
A complete proof of Lemma \ref{lemma: lambda max sum of delta deltaT bound} can be found at \cref{proof of lemma: lambda max sum of delta deltaT bound}. 
\begin{proof} (sketch)
We first define a ``stable'' event $E_{\delta}$, which holds with probability $1-\delta$, on which for large enough $T$, the estimation errors are uniformly bounded by some small constant. Intuitively, $E_{\delta}$ is the event on which the system remains well-behaved, in the sense that the system is always well controlled after certain time, which makes our analysis much easier. Lemma~\ref{lemma: favorable event} defines $E_\delta$ and proves that it holds with high probability; its proof is deferred to Appendix~\ref{proof of lemma: favorable event}, but it basically follows from a union bound applied to \cref{eq: naive theta bound with log}.
\begin{lemma}
\label{lemma: favorable event}
\cref{alg:myAlg} applied to a system described by \cref{eq:system eq} under Assumption \ref{asm:InitialStableCondition} satisfies, for fixed $\epsilon_0 \lesssim 1$ and any
$\delta > 0$,
\begin{equation}
\label{eq: favorable event}
    E_\delta := \mathset{\lnorm{\hat\Theta_T - \Theta}, \lnorm{\Kh_T - K} \le \epsilon_0
\text{, for all } T \gtrsim  \log^2(1/\delta)}, \PP{E_\delta} \ge 1-\delta.
\end{equation} 
\end{lemma}
Then starting with the inequality
\begin{equation*}
    \norm{\Delta_t}^2 = \norm{(\Kh_t - K)x_t + \eta_t}^2 \lesssim \norm{\Kh_t - K}^2\norm{x_t}^2 + \norm{\eta_t}^2,
\end{equation*}
we can show that with probability $1-\delta$:
\begin{align*}
    \maxEig{\sum_{t=0}^{T-1}\Delta_t\Delta_t^\top 1_{E_{\delta}}}  \lesssim 1/\delta 
    \left(
    \sum_{t=1}^{T-1}\EE{t^{1/2}\norm{\Kh_t - K}^4 +  t^{-1/2}\norm{x_t}^4 1_{E_{\delta}}
    }  + T^{1/2} 
    \right).
\end{align*}
%

We prove in Lemma \ref{lemma: xt norm bound} that for $t \gtrsim \log^2(1/\delta)$, $\EE{\norm{x_t}^41_{E_{\delta}}} \lesssim 1$. For $t \lesssim \log^2(1/\delta)$ we have the bound $\E{\norm{x_t}^2} \lesssim \log^2(t)$ from Eq. (104) of \citet{wang2020exact}. We show the same proof applies if we increase the exponent from 2 to 4:
\begin{equation}
    \label{eq: norm xt 4 bound}
    \E{\norm{x_t}^4} \lesssim \log^4(t).
\end{equation}
Applying these bounds produces
\begin{align*}
\maxEig{\sum_{t=0}^{T-1}\Delta_t\Delta_t^\top 1_{E_{\delta}}}   \lesssim 1/\delta 
    \left(
    \sum_{t=1}^{T-1}\EE{t^{1/2}\norm{\Kh_t - K}^4} + 
    \log^2(1/\delta) +
    T^{1/2}  
    \right),
\end{align*}
and we finish the proof by removing $1_{E_{\delta}}$ on the left hand side and decreasing the probability with which the inequality holds from $1-\delta$ to $1-2\delta$.
\end{proof}


\begin{lemma}
\label{lemma: dtheta recursive}
\cref{alg:myAlg} applied to a system described by \cref{eq:system eq} under Assumption \ref{asm:InitialStableCondition} satisfies, for any $0 < \delta < 1/2$ and $T \gtrsim \log^3(1/\delta)$,

\begin{align}
\label{eq: dtheta recursive eq}
\begin{split}
    &\P\Bigg[
    T^{1/2}\lnorm{\hat\Theta_T-\Theta}^2 \gtrsim 
    \logg{T^{-1/2}\left(
    \sum_{t=1}^{T-1}\EE{t^{1/2}\norm{\Kh_t - K}^4} 
    \right) + 1 } + 
    \log\left(\frac{1}{\delta}\right)
    \Bigg] \le  3\delta.   
\end{split}
\end{align}

\end{lemma}
A complete proof of Lemma \ref{lemma: dtheta recursive} can be found at \cref{subsection: proof of lemma: dtheta recursive}.
\begin{proof} (sketch)
Combining Lemma \ref{lemma: lambda max sum of delta deltaT bound} and \cref{eq: theta bound with lambda max deltat},
\begin{align*}
\begin{split}
    &\P\Bigg[
    T^{1/2}\lnorm{\hat\Theta_T-\Theta}^2 \gtrsim 
    \logg{T^{-1/2}\left(
    \sum_{t=1}^{T-1}\EE{t^{1/2}\norm{\Kh_t - K}^4} + 
    \log^2(1/\delta)
    \right) + 1 } + 
    \log\left(\frac{1}{\delta}\right)
    \Bigg]\\
    &\hspace{14.3cm}\le  3\delta.
\end{split}
\end{align*}
Then we show that the $\log^2(1/\delta)$ can be moved outside and merged with the $\log(1/\delta)$ term:
\begin{align*}
    & \logg{T^{-1/2}\left(
    \sum_{t=1}^{T-1}\EE{t^{1/2}\norm{\Kh_t - K}^4} + 
    \log^2(1/\delta)
    \right) + 1 } + 
    \log\left(\frac{1}{\delta}\right) 
    \\
    & \quad \lesssim
    \logg{
        T^{-1/2}
        \left(
            \sum_{t=1}^{T-1}\EE{t^{1/2}\norm{\Kh_t - K}^4}  
        \right) +  1 
    } + 
    \log\left(\frac{1}{\delta}\right),
\end{align*}
completing the proof.
\end{proof}

We are now able to state the main result of this section:
\begin{theorem}
\label{theorem: theta bound without log}
\cref{alg:myAlg} applied to a system described by \cref{eq:system eq} under Assumption \ref{asm:InitialStableCondition} satisfies
\begin{equation}
\label{eq: theta bound without log}
    \lnorm{\hat\Theta_T-\Theta} = O_p(T^{-1/4})\text{ and }\lnorm{\Kh_T-K} = O_p(T^{-1/4}).
\end{equation}
\end{theorem}
A complete proof of \cref{theorem: theta bound without log} can be found at \cref{subsection: lemma: theta bound without log}.
\begin{proof}
By Proposition 4 of \citet{simchowitz2020naive}, 
\begin{equation}
\label{eq: delta K controlled by delta theta}
    \norm{\Kh_T - K} \lesssim \norm{\hat\Theta_T - \Theta}.
\end{equation}
as long as $\norm{\hat\Theta_T - \Theta} \le \epsilon_0$, where $\epsilon_0$ is some fixed constant determined by the system parameters. We want to focus on cases where $\norm{\hat\Theta_T - \Theta} \le \epsilon_0$ to transfer $T^{1/2}\lnorm{\hat\Theta_T-\Theta}^2$ to $T^{1/2}\lnorm{\Kh_T-K}^2$ so that Lemma \ref{lemma: dtheta recursive} has only estimation error of $\Kh_T$.

We can estimate $\EE{T\lnorm{\hat\Theta_T-\Theta}^4 1_{\norm{\hat\Theta_T-\Theta} \le \epsilon_0}}$ by calculating the integral using the tail bound from Lemma \ref{lemma: dtheta recursive} as long as $T \gtrsim \log^3(1/\delta)$. The further tails can be bounded by the constant $\epsilon_0$. 

As a result, when $T \ge T_0$ ($T_0$ is a large enough constant so that $3e^{-cT^{1/3}}T\epsilon_0^4 \le 1$):
\begin{align*}
    &\EE{T\lnorm{\hat\Theta_T-\Theta}^4 1_{\norm{\hat\Theta_T-\Theta} \le \epsilon_0}} \\
    & \quad \lesssim \left(
        \log\left(T^{-1/2}\left(\sum_{t=1}^{T-1}t^{-1/2} \EE{t\lnorm{\Kh_t-K}^4} \right) \right)  + 1
    \right)^2 + 3e^{-cT^{1/3}}T\epsilon_0^4 \\
    & \quad \lesssim \left(
        \log\left(T^{-1/2}\left(\sum_{t=1}^{T-1}t^{-1/2} \EE{t\lnorm{\Kh_t-K}^4} \right) \right)  + 1
    \right)^2 + 1.
\end{align*}
On the right-hand side, consider the maximum of $\EE{t\lnorm{\Kh_t-K}^4}$ from $T_0$ to $T_{\max} \ge T$,
\begin{align*}
    &\EE{T\lnorm{\hat\Theta_T-\Theta}^4
    1_{\norm{\hat\Theta_T-\Theta} \le \epsilon_0}
    } \\
    & \quad 
    \text{(\cref{alg:myAlg} ensures $\norm{\Kh_t} \le C_K$)}
    \\
    & \quad \lesssim \left(
        \log\left(T^{-1/2}\left(\sum_{t=1}^{T_0}t^{-1/2} (C_K + \norm{K})^2
        \right)+ \max_{T_0 \le s \le T_{\max}} \EE{s\lnorm{\Kh_s-K}^4} \right)  + 1
    \right)^2 + 1  \\
    &\quad \lesssim \left(
        \log\left(T^{-1/2}T_0^{1/2}\cdot 1+ \max_{T_0 \le s \le T_{\max}} \EE{s\lnorm{\Kh_s-K}^4} \right)  + 1
    \right)^2 + 1 \\
    &\quad \lesssim 
    \left(
        \log\left(1+ \max_{T_0 \le s \le T_{\max}} \EE{s\lnorm{\Kh_s-K}^4} \right)  + 1
    \right)^2 + 1 
\end{align*}
By \cref{eq: delta K controlled by delta theta}, we can transfer $\lnorm{\hat\Theta_T-\Theta}$ on the left hand side to $\lnorm{\Kh_T-K}$ as long as $\lnorm{\hat\Theta_T-\Theta} \le \epsilon_0$.
By Lemma \ref{lemma: delta theta bound}, the probability $\delta$ that $\lnorm{\hat\Theta_T-\Theta} \le \epsilon_0$ does not hold can be solved from 
\begin{equation*}
    T^{-1/4}\sqrt{
    \left(
    \log T + \log\left(\frac{1}{\delta}\right)
    \right)} = \epsilon_0,
\end{equation*}
which gives 
\begin{equation*}
    \delta = T e^{-\epsilon_0^2T^{1/2}}.
\end{equation*}
As a result (all $T\ge T_0$ satisfies $T e^{-\epsilon_0^2T^{1/2}}T(C_K + \norm{K})^4 \le 1$)
\begin{align*}
    &\EE{T\lnorm{\hat\Theta_T-\Theta}^4
     1_{\norm{\hat\Theta_T-\Theta} \le \epsilon_0}
    }  \gtrsim 
    \EE{T\lnorm{\Kh_T-K}^4
     1_{\norm{\hat\Theta_T-\Theta} \le \epsilon_0}
    }  \\
    & \quad \ge \EE{T\lnorm{\Kh_T-K}^4} - T e^{-\epsilon_0^2T^{1/2}}T(C_K + \norm{K})^4 \\
    & \quad \ge \EE{T\lnorm{\Kh_T-K}^4} - 1.
\end{align*}
Now we have 
\begin{equation*}
\EE{T\lnorm{\Kh_T-K}^4} \lesssim 
        \left(
        \log\left(1+ \max_{T_0 \le s \le T_{\max}} \EE{s\lnorm{\Kh_s-K}^4} \right)  + 1.
    \right)^2 + 1 
\end{equation*}
The right hand side is constant. Taking the maximum over $T$ from $T_0$ to $T_{\max}$ on the left hand side:
\begin{equation*}
    \max_{T_0 \le s \le T_{\max}} \EE{s\lnorm{\Kh_s-K}^4}
    \lesssim \left(
        \log\left(1+ \max_{T_0 \le s \le T_{\max}} \EE{s\lnorm{\Kh_s-K}^4} \right)  + 1
    \right)^2 + 1 
\end{equation*}
Thus 
\begin{equation*}
    \max_{T_0 \le s \le T_{\max}} \EE{s\lnorm{\Kh_s-K}^4} \lesssim 1.
\end{equation*}
The hidden constant only depends on $T_0$, and hence the same inequality holds for \emph{any} $T_{\max}$:
\begin{equation*}
    \max_{s \ge T_0} \EE{s\lnorm{\Kh_s-K}^4}  
    \lesssim 1.
\end{equation*}
Plugging this back to \cref{eq: dtheta recursive eq} gives that when $T \gtrsim \log^3(1/\delta)$,
\begin{equation*}
    \P\left[
    T^{1/2}\lnorm{\hat\Theta_T-\Theta}^2 \gtrsim 
    \logg{T^{-1/2}\left(\sum_{t=1}^{T_0}t^{-1/2} \EE{t\lnorm{\Kh_t-K}^4}
    \right)+ 1} + \log\left(\frac{1}{\delta} \right) 
    \right] \le  3\delta.
\end{equation*}
Because $\norm{\Kh_t} \le C_K$, the sum over the first $T_0$ terms is of negligible order, so that the above equation can be simplified to
\begin{equation*}
    \P\left[
    T^{1/2}\lnorm{\hat\Theta_T-\Theta}^2 \gtrsim 
    \log\left(\frac{1}{\delta} \right) 
    \right] \le  3\delta.
\end{equation*}
Thus,
\begin{equation*}
    \lnorm{\hat\Theta_T-\Theta} = O_p(T^{-1/4}),
\end{equation*}
and $\lnorm{\Kh_T-K} = O_p(T^{-1/4})$ is a direct corollary from \cref{eq: delta K controlled by delta theta}.
\end{proof}

\section{Bounding the regret by $\bigOp{\sqrt{T}}$}
\label{sec: proof_reg}
We start this section by stating the main result of this paper, our regret upper-bound that exactly rate-matches the regret lower-bound of $\bigOmega{\sqrt{T}}$ established in \citet{simchowitz2020naive}.
\begin{theorem}
\label{theorem:regret}
\cref{alg:myAlg} applied to a system described by \cref{eq:system eq} under Assumption \ref{asm:InitialStableCondition} satisfies
\begin{equation}
    \label{eq: tight regret}
    \mathcal{R}(U,T) = \bigOp{\sqrt{T}}.
\end{equation}
\end{theorem}
A complete proof of \cref{theorem:regret} can be found at \cref{subsection: proof of theorem:regret}.
\begin{proof} (sketch)
Our first step is to show the following lemma bounding the cumulative costs $\mathcal{J}$ of the system under \cref{alg:myAlg} and under the oracle optimal controller. 
\begin{lemma}
\label{lemma: cost for our algo and optimal algo}
\cref{alg:myAlg} applied to a system described by \cref{eq:system eq} under Assumption \ref{asm:InitialStableCondition} satisfies, 
\begin{equation*}
        \mathcal{J}(U,T)= \sum_{t=1}^{T} \tilde\varepsilon_{t}^\top P \tilde\varepsilon_{t} +
    \sum_{t=1}^{T} \eta_{t}^\top R \eta_{t} + \bigOp{T^{1/2}}
\end{equation*}
and 
\begin{equation*}
        \mathcal{J}(U^*,T)= \sum_{t=1}^{T} \varepsilon_{t}^\top P \varepsilon_{t} + \bigOp{T^{1/2}},
\end{equation*}
where $\varepsilon_t$ is the system noise and $\eta_t$ is the exploration noise in \cref{alg:myAlg}, and $\tilde{\varepsilon}_t = B\eta_t + \varepsilon_t$. 
\end{lemma}
Before sketching the proof of Lemma \ref{lemma: cost for our algo and optimal algo} (a complete proof can be found in \cref{subsection: proof of lemma: cost for our algo and optimal algo}), we show how to finish the proof of \cref{theorem:regret} with just a few more steps:
\begin{align*}
    &\mathcal{R}(U,T) = \mathcal{J}(U,T)- \mathcal{J}(U^*,T) \\
    &\quad = \sum_{t=1}^{T} \tilde\varepsilon_{t}^\top P \tilde\varepsilon_{t} +
    \sum_{t=1}^{T} \eta_{t}^\top R \eta_{t} - \sum_{t=1}^{T} \varepsilon_{t}^\top P \varepsilon_{t} + \bigOp{T^{1/2}}\\
    &\quad = 2\sum_{t=1}^{T} \varepsilon_{t}^\top P (B\eta_t) + \sum_{t=1}^{T} (B\eta_t)^\top P (B\eta_t) + 
    \sum_{t=1}^{T} \eta_{t}^\top R \eta_{t} + \bigOp{T^{1/2}}.
\end{align*}
The final result follows by bounding the three summations in the last line by $\bigOp{T^{1/2}}$:
because $\eta_t = \bigOp{t^{-1/4}}$, the quadratic summations $\sum_{t=1}^{T} (B\eta_t)^\top P (B\eta_t)$ and $\sum_{t=1}^{T} \eta_{t}^\top R \eta_{t}$ are both of order $\bigOp{T^{1/2}}$ and the cross term $2\sum_{t=1}^{T} \varepsilon_{t}^\top P (B\eta_t) = o_p\left(T^{1/2}\right)$.
\end{proof}

\begin{proof} (sketch for Lemma \ref{lemma: cost for our algo and optimal algo})
We only prove the first equation because the second equation is a special case of the first equation (with $\eta_t = 0$ and $\Kh_t = K$).  The idea is to consider a new system with system noise $\tilde{\varepsilon}_t = B\eta_t + \varepsilon_t$ and controller $\tilde{u}_t = \Kh_t x_t$. One can show that the new system shares the same states $x_t$ as the original system and the cost in the new system is:
\begin{equation}
\label{eq: new system cost}
    \sum_{t=1}^{T} x_t^\top Qx_t + \tilde{u}_t^\top R\tilde{u}_t = \sum_{t=1}^{T} \tilde\varepsilon_{t}^\top P \tilde\varepsilon_{t} +
    \bigOp{T^{1/2}}
.\end{equation}
The difference between the original cost and transformed cost is 
\begin{equation}
\label{eq: difference action new and og system}
    \sum_{t=1}^{T} u_t^\top  Ru_t -  \tilde{u}_t^\top R\tilde{u}_t  =
    \sum_{t=1}^T \eta_t^\top R \eta_t + 
    o\left(T^{1/4}\log^{\frac{3}{2}}(T) \right) \as
\end{equation}
The result of the Lemma follows by summing \cref{eq: new system cost,eq: difference action new and og system}; we now briefly sketch the proofs of each equation. \cref{eq: new system cost,eq: difference action new and og system} are stated in the Appendix as Lemmas~\ref{lemma: new sys cost} and \ref{lemma: trans cost} in the Appendix and their complete proofs are given in \cref{subsection: Cost of new system,subsection: Cost difference induced by transformation}. 

\paragraph{\cref{eq: new system cost}:} 
After some substitutions to leverage an identity from Lemma 18 of \citet{wang2020exact} and applying bounds to straightforward terms, the cost of the new system can be written as
\begin{align*}
    &\sum_{t=1}^{T} x_t^\top Qx_t + \tilde{u}_t^\top R\tilde{u}_t \\
    & \quad = \sum_{t=1}^{T} \bigg[x_t^\top (\Kh_t-K)^\top( R + B^\top P B) (\Kh_t-K)x_t  + 2  \tilde\varepsilon_{t}^\top P (A+B\Kh_t)x_t + \tilde\varepsilon_{t}^\top P \tilde\varepsilon_{t}\bigg]
        + \logOO_p(1).
\end{align*}
By \cref{theorem: theta bound without log}, $\norm{\Kh_t - K}=\bigOp{T^{-1/4}}$, and we can show essentially that $x_t$ is of constant order. This gives that the first sum is of order $\sum_{t=1}^{T} x_t^\top (\Kh_t-K)^\top( R + B^\top P B) (\Kh_t-K)x_t = \bigOp{T^{1/2}}$. By noting that $\tilde\varepsilon_{t} \independent P(A+B\Kh_t)x_t$ and both $\tilde\varepsilon_{t}$ and $P(A+B\Kh_t)x_t$ are of constant order, we can use standard properties of martingales to show $\sum_{t=1}^{T} \tilde\varepsilon_{t}^\top P (A+B\Kh_t)x_t = \bigOp{T^{1/2}}$ as well. \cref{eq: new system cost} is then just the third sum plus the combination of the aforementioned bounds.

\paragraph{\cref{eq: difference action new and og system}:}
The difference is expressed as
\begin{equation*}
\begin{split}
        \sum_{t=1}^{T} u_t^\top  Ru_t -  \tilde{u}_t^\top R\tilde{u}_t     
    =&  \sum_{t=1}^{T}(\Kh_t x_t + \eta_t)^\top  R (\Kh_t x_t + \eta_t)   - \sum_{t=1}^{T} (\Kh_t x_t)^\top  R (\Kh_t x_t) \\
    =& 2\sum_{t=1}^{T} (\Kh_t x_t)^\top  R \eta_t +  \sum_{t=1}^{T} \eta_t^\top R \eta_t, 
\end{split}
\end{equation*}
and we simply bound the first term by Eq.~(83) of \citet{wang2020exact},
\begin{equation*}
    \sum_{t=1}^{T} (\Kh_t x_t)^\top  R \eta_t = o\left(T^{1/4}\log^{\frac{3}{2}}(T) \right) \as,
\end{equation*}
which completes the proof.
\end{proof}

\section{Discussion}
Before we can fully understand the practical performance of RL and deploy it in real-world, high-stakes environments, we need to at least understand it well in the simplest, most structured problem settings. This paper provides progress in that direction by,
for the LQR problem with unknown dynamics, proving the first regret upper-bound of $O_p(\sqrt{T})$,
exactly matching the rate of the best-known lower-bound of $\Omega_p(\sqrt{T})$ established in \citet{simchowitz2020naive}. There are related settings 
such as non-linear LQR \citep{kakade2020information} and non-stationary LQR \citep{luo2021dynamic} whose best known regret upper-bounds are $O_p(\sqrt{T}\,\text{polylog}(T))$, and we hope our work can shed light on removing the $\text{polylog}(T)$ terms in these settings as well. Finally, for the practical deployment of RL algorithms, the constant factor multiplying the regret rate really matters, so it is our hope that now that the LQR rate is tightly characterized the field can move on to characterizing and tightening the constant in the optimal regret, which we expect will lead to algorithmic innovation as well. 

\section*{Acknowledgement}
The authors are grateful for partial support from NSF CBET--2112085.

\bibliography{all_bib}

\appendix 
\section{Proofs from Section~\ref{sec: proof_est}}

\subsection{Proof of Lemma \ref{lemma: delta theta bound}}
\label{subsection: proof of lemma: delta theta bound}

\begin{lemma*}[Bound with log term]
\cref{alg:myAlg} applied to a system described by \cref{eq:system eq} under Assumption \ref{asm:InitialStableCondition} satisfies, when $0 < \delta < 1/2$, for any $T \gtrsim \log(1/\delta)$,
\begin{align*}
\P\left[\lnorm{\hat\Theta_T-\Theta} \gtrsim T^{-1/4}\sqrt{
\left(
\log T + \log\left(\frac{1}{\delta}\right)
\right)} 
\right] \le  \delta.
\end{align*}
As a direct corollary:
\begin{equation*}
    \norm{\hat\Theta_T - \Theta} = \bigOp{T^{-1/4}\log^{1/2}(T)}.
\end{equation*}
\end{lemma*}

\begin{proof}
If we can tolerate the additional log terms, the way of upper bounding the Gram matrix is to use the result from Eq. (104) of \citet{wang2020exact}: $\E{\norm{z_t}^2} \lesssim \log^2(t)$. 
Inspired by the equation before Proposition 3.1 of \citet{simchowitz2018learning}, we have:
for any random positive semi-definite matrix $M \in \real^{d_M \times d_M}$ with $\EE{M} \succ 0$, 
\begin{equation*}
\begin{split}
    &\PP{M \npreceq \frac{d_M}{\delta}\EE{M}} \\
    &\quad = \PP{\lambda_{\max}\left(
    (\EE{M})^{-1/2}M(\EE{M})^{-1/2} \right)
    \ge \frac{d_M}{\delta} 
    } \\
    &\quad \le 
    \EE{
        \lambda_{\max}
        \left((\EE{M})^{-1/2}M(\EE{M})^{-1/2} \right)
    }
    / \frac{d_M}{\delta} \\
    &\quad \le 
    \EE{
        \Tr
        \left((\EE{M})^{-1/2}M(\EE{M})^{-1/2} \right)
    }
    / \frac{d_M}{\delta} \\
    &\quad = 
        \Tr
        \left((\EE{M})^{-1/2}\EE{M}(\EE{M})^{-1/2} \right)
    / \frac{d_M}{\delta} \\
    &\quad = 
        \Tr \left(I_{d_M}\right)
    / \frac{d_M}{\delta} \\
    &\quad = \delta,
\end{split}
\end{equation*}
which means 
\begin{equation}
\label{eq: matrix M upper bound}
\PP{M \preceq \frac{d_M}{\delta}\EE{M}} \ge 1-\delta.
\end{equation}
Take $M = \sum_{t=0}^{T-1} z_tz_t^\top$:
\begin{equation*}
    \PP{\sum_{t=0}^{T-1} z_tz_t^\top \preceq \frac{n+d}{\delta}\sum_{t=0}^{T-1} \E z_tz_t^\top} \ge 1-\delta.
\end{equation*}
On the other hand, we have:
\begin{equation*}
     \sum_{t=0}^{T-1} \E z_tz_t^\top \preceq \sum_{t=0}^{T-1}\E \norm{z_t}^2 I_{n+d} \precsim T\log^2(T) I_{n+d}.
\end{equation*}
As a result, we can take $\bar\Gamma \eqsim  \log^2(T) I_{n+d}/\delta$. 

By Theorem 2.4 of \citet{simchowitz2018learning}, the lower bound condition $\PP{\sum_{t=0}^{T-1} z_tz_t^\top \succeq T\Gamlow} \ge 1- \delta$ in \cref{eq: simchowitz conclusion} could be replaced by the $(k, \Gamlow, p)$-BMSB condition on $\{z_t\}_{t \ge 1}$.
\begin{defn}[BMSB condition from \citep{simchowitz2018learning}]\label{def:bmsb}  
	    
	     Given an $\{\calF_t\}_{t \ge 1}$-adapted random process $\{x_t\}_{t \ge 1}$ taking values in $\R^d$, we say that it satisfies the $(k,\Gamlow,p)$-matrix block martingale small-ball (BMSB) condition for $\Gamlow \succ 0$ if, for any unit vector $w$ and $j \ge 0$, $\frac{1}{k}\sum_{i=1}^{k} \P( |\langle w, x_{j+i}\rangle | \ge \sqrt{w^\top \Gamlow w} | \calF_{j}) \ge p \as$ 
    \end{defn}

The BMSB condition ensures a lower bound on the independent randomness in each entry of an adapted sequence $\{x_t\}_{t \ge 1}$ given all past history. 
%
By Lemma 15 of \citet{wang2020exact}, the process $\{z_t\}_{t=0}^{T-1}$ satisfies the 
\begin{equation*}
    (k, \Gamlow, p) = \left(1, \sigma_\eta^2 T^{-1/2} \min\left(\frac12, \frac{\sigma_\varepsilon^2}{2\sigma_\varepsilon^2C_K^2 + \sigma_\eta^2}\right)I_{n+d},
    \frac{3}{10}\right) \text{BMSB condition.}
\end{equation*}
Thus $\bar\Gamma\Gamlow^{-1} \eqsim \sqrt{T}\log^2(T)I_{n+d}$
Then \cref{eq: simchowitz conclusion} becomes 
\begin{align*}
\P\left[\lnorm{\hat\Theta_T-\Theta} \gtrsim \sqrt{\frac{1 + 
\log \det (\log^2(T)\sqrt{T}I_{n+d}/\delta) + \log\left(\frac{1}{\delta}\right)}
{T^{1/2}}} 
\right] \le  \delta.
\end{align*}
Here $\log^2(T)\sqrt{T}$ is dominated by $T$ when $T$ is large enough. Also we can hide the constant $1$ because $\delta < 1/2$ which implies $\log(1/\delta) \gtrsim 1$. Thus,
\begin{align*}
\P\left[\lnorm{\hat\Theta_T-\Theta} \gtrsim \sqrt{\frac{ \log(T) + \log\left(\frac{1}{\delta}\right)}
{T^{1/2}}} 
\right] \le  \delta.
\end{align*}
The condition for the above equation to hold is \cref{eq: T condition simchowitz}:
\begin{align*}
    T &\gtrsim  \log \left(\frac{1}{\delta}\right) + 1 +  \log \det (\Gambar \Gamlow^{-1}) \\
    & \quad \text{(because $\delta < 1/2$ we can hide $1$)} \\
    &\gtrsim  \log \left(\frac{1}{\delta}\right) +  \log \det (\Gambar \Gamlow^{-1})
    ,
\end{align*}
which by $\bar\Gamma\Gamlow^{-1} \eqsim \sqrt{T}\log^2(T)I_{n+d}$ becomes
\begin{equation*}
    T \gtrsim \log(1/\delta) + \log(T).
\end{equation*}
This condition can be simplified to $T \gtrsim \log(1/\delta)$ because $T$ dominates $\log(T)$.
\end{proof}

\subsection{Proof of Lemma \ref{lemma: better upper and lower bounds of GT}}
\label{Proof of lemma: better upper and lower bounds of GT}
\begin{lemma*}
\cref{alg:myAlg} applied to a system described by \cref{eq:system eq} under Assumption \ref{asm:InitialStableCondition} satisfies, for any $0 < \delta < 1/2$ and $T \gtrsim \log^3(1/\delta)$, with probability at least $1-\delta$:
\begin{align}
\label{eq: GT upper and lower bounds*}
\begin{split}
    &
    T\Gamlow :=
    \begin{bmatrix}
            I_n \\
            K
        \end{bmatrix} T 
        \begin{bmatrix}
            I_n \\
            K
        \end{bmatrix}^\top + 
        \begin{bmatrix}
-K^\top \\
I_d
\end{bmatrix} T^{1/2}
\begin{bmatrix}
-K^\top \\
I_d
\end{bmatrix}^\top \\
& \qquad \precsim 
\sum_{t=0}^{T-1} z_tz_t^\top
\precsim \left(
    \frac1\delta 
    \begin{bmatrix}
I_n \\
K
\end{bmatrix} T 
\begin{bmatrix}
I_n \\
K
\end{bmatrix}^\top + 
\begin{bmatrix}
-K^\top \\
I_d
\end{bmatrix} \maxEig{ \sum_{t=0}^{T-1}\Delta_t\Delta_t^\top}
\begin{bmatrix}
-K^\top \\
I_d
\end{bmatrix}^\top 
\right)
:= T\Gambar,
\end{split}
\end{align}
\end{lemma*}
where $\Delta_t := (\Kh_t - K)x_t + \eta_t$.
\begin{proof}
Consider the matrix
\begin{equation}
\label{eq: GT definition}
\begin{split}
    G_T &:= \sum_{t=0}^{T-1} z_tz_t^\top \\
    &=\sum_{t=0}^{T-1} 
        \begin{bmatrix}
        x_t \\
        u_t
        \end{bmatrix}
        \begin{bmatrix}
        x_t \\
        u_t
        \end{bmatrix}^\top \\
    &=
        \sum_{t=0}^{T-1} 
        \begin{bmatrix}
        x_t \\
        Kx_t + (\Kh_t-K)x_t + \eta_t
        \end{bmatrix}
        \begin{bmatrix}
        x_t \\
        Kx_t + (\Kh_t-K)x_t + \eta_t
        \end{bmatrix}^\top.
\end{split}
\end{equation}
Let
$$\Delta_t = (\Kh_t-K)x_t + \eta_t,$$
and then the previous equation can be written as 
\begin{align*}
G_T = \sum_{t=0}^{T-1} 
z_t z_t^\top = 
\begin{bmatrix}
I \\
K
\end{bmatrix}
\sum_{t=0}^{T-1} 
x_t x_t^\top
\begin{bmatrix}
I \\
K
\end{bmatrix}^\top +
    \sum_{t=0}^{T-1}
    \begin{bmatrix}
    0_n &  x_t \Delta_t^\top\\
    \Delta_tx_t^\top  & \Delta_t\Delta_t^\top + K x_t \Delta_t^\top + \Delta_tx_t^\top K^\top  \\
    \end{bmatrix}.
\end{align*}

We consider the dominating part $\begin{bmatrix}
I \\
K
\end{bmatrix}
\sum_{t=0}^{T-1} 
x_t x_t^\top
\begin{bmatrix}
I \\
K
\end{bmatrix}^\top $
(smallest eigenvalue scales with $T$) and the remainder part 
$\sum_{t=0}^{T-1}
    \begin{bmatrix}
    0_n &  x_t \Delta_t^\top\\
    \Delta_tx_t^\top  & \Delta_t\Delta_t^\top + K x_t \Delta_t^\top + \Delta_tx_t^\top K^\top  \\
    \end{bmatrix}
$ separately. This separation is useful because the dominating part purely lies in the subspace spanned by $\begin{bmatrix}
I \\
K
\end{bmatrix}$, and the remainder part is of smaller order than $T$. For the dominating part we have 

\begin{lemma}
\label{lemma: large eigenvalue part}
\cref{alg:myAlg} applied to a system described by \cref{eq:system eq} under Assumption \ref{asm:InitialStableCondition} satisfies, for any $0 < \delta < 1/2$ and $T \gtrsim \log^3(1/\delta)$, with probability at least $1-\delta$:
\begin{equation}
\label{eq: upper lower bound of the main part in gram matrix*}
\begin{bmatrix}
I \\
K
\end{bmatrix} T 
\begin{bmatrix}
I \\
K
\end{bmatrix}^\top
\preceq
\begin{bmatrix}
I \\
K
\end{bmatrix}
\sum_{t=0}^{T-1} 
x_t x_t^\top
\begin{bmatrix}
I \\
K
\end{bmatrix}^\top \preceq 
1/\delta
\begin{bmatrix}
I \\
K
\end{bmatrix} T 
\begin{bmatrix}
I \\
K
\end{bmatrix}^\top.
\end{equation} 

\end{lemma}
The proof can be found at \cref{subsection: proof of lemma: large eigenvalue part}. By Lemma 34 of \citet{wang2020exact}, the process $(z_t)_{t = 0}^{T-1}$ satisfies the $(1, \sigma_\eta^2 T^{-1/2}\min\left(\frac12, \frac{\sigma_\varepsilon^2}{2\sigma_\varepsilon^2C_K^2 + \sigma_\eta^2}\right)I_{n+d}, \frac3{10})$-BMSB condition, which guarantees us a lower bound $G_T = \sum_{t=0}^{T-1}z_tz_t^\top \gtrsim T^{1/2}I_{n+d}$. Also by the left hand side of Lemma \ref{lemma: large eigenvalue part},
\begin{equation*}
    G_T \succeq 
    \begin{bmatrix}
    I \\
    K
    \end{bmatrix}
    \sum_{t=0}^{T-1} 
    x_t x_t^\top
    \begin{bmatrix}
    I \\
    K
    \end{bmatrix}^\top
    \succeq 
    \begin{bmatrix}
    I \\
    K
    \end{bmatrix} T 
    \begin{bmatrix}
    I \\
    K
    \end{bmatrix}^\top.
\end{equation*}
Combining these two equations we have:

\begin{lemma}[Lower bound of $G_T$]
\label{lemma: lower bound of GT}
\cref{alg:myAlg} applied to a system described by \cref{eq:system eq} under Assumption \ref{asm:InitialStableCondition} satisfies, when $0 < \delta < 1/2$, for any $T \gtrsim \log^3(1/\delta)$, with probability at least $1-\delta$:
\begin{equation}
\label{eq: G_T lower bound}
    G_T  \succsim 
        \begin{bmatrix}
            I_n \\
            K
        \end{bmatrix} T
        \begin{bmatrix}
            I_n \\
            K
        \end{bmatrix}^\top + 
        T^{1/2} I_{n+d}.
\end{equation}
\end{lemma}
See the proof at \cref{subsection: Proof of lemma: lower bound of GT}. At first glance this seems wrong because the right hand side of \cref{eq: G_T lower bound} does not have $\delta$. Actually, the role of $\delta$ is present in the constraint $T \gtrsim \log^3(1/\delta)$. The result is not surprising because as $T$ grows larger it becomes exponentially unlikely that $G_T$ can be smaller than, for example, $\frac12 \E{G_T}$.

\begin{lemma}[Upper bound of $G_T$]
\label{lemma: upper bound of $G_T$}
\cref{alg:myAlg} applied to a system described by \cref{eq:system eq} under Assumption \ref{asm:InitialStableCondition} satisfies, when $0 < \delta < 1/2$, for any $T \gtrsim \log^3(1/\delta)$, with probability at least $1-\delta$:
\begin{equation}
\label{eq: G_T upper bound}
    G_T \precsim \left(
    \frac1\delta 
    \begin{bmatrix}
I_n \\
K
\end{bmatrix} T
\begin{bmatrix}
I_n \\
K
\end{bmatrix}^\top + 
\begin{bmatrix}
-K^\top \\
I_d
\end{bmatrix} \maxEig{ \sum_{t=0}^{T-1}\Delta_t\Delta_t^\top}
\begin{bmatrix}
-K^\top \\
I_d
\end{bmatrix}^\top
\right).
\end{equation}
\end{lemma}
Proof can be found at \cref{subsection: Proof of lemma: upper bound of $G_T$}.

\cref{eq: GT upper and lower bounds*} is a direct corollary of \cref{eq: G_T lower bound} and \cref{eq: G_T upper bound}. 
\begin{equation*}
\begin{split}
&
    \begin{bmatrix}
I_n \\
K
\end{bmatrix} T
\begin{bmatrix}
I_n \\
K
\end{bmatrix}^\top + 
\begin{bmatrix}
-K^\top \\
I_d
\end{bmatrix} T^{1/2}
\begin{bmatrix}
-K^\top \\
I_d
\end{bmatrix}^\top
 \\
     & \qquad \precsim \begin{bmatrix}
            I_n \\
            K
        \end{bmatrix} T
        \begin{bmatrix}
            I_n \\
            K
        \end{bmatrix}^\top + 
        T^{1/2} I_{n+d}  \\
        & \qquad \precsim G_T \\
        & \qquad \precsim 
    \frac1\delta 
    \begin{bmatrix}
I_n \\
K
\end{bmatrix} T 
\begin{bmatrix}
I_n \\
K
\end{bmatrix}^\top + 
\begin{bmatrix}
-K^\top \\
I_d
\end{bmatrix} \maxEig{ \sum_{t=0}^{T-1}\Delta_t\Delta_t^\top}
\begin{bmatrix}
-K^\top \\
I_d
\end{bmatrix}^\top.
\end{split}
\end{equation*}
The first inequality is because
\begin{equation*}
\alpha^\top
\left(
    \norm{\begin{bmatrix}
-K^\top \\
I_d
\end{bmatrix}}^2I_{n+d} - 
\begin{bmatrix}
-K^\top \\
I_d
\end{bmatrix} 
\begin{bmatrix}
-K^\top \\
I_d
\end{bmatrix}^\top
\right)
\alpha
\ge 0.
\end{equation*}
for any $\alpha \in \real^{n+d}$.
\end{proof}

\subsection{Proof of Lemma \ref{lemma: lambda max sum of delta deltaT bound}}
\label{proof of lemma: lambda max sum of delta deltaT bound}

\begin{lemma*}
\cref{alg:myAlg} applied to a system described by \cref{eq:system eq} under Assumption \ref{asm:InitialStableCondition} satisfies, for any $0 < \delta < 1/2$ and $T \gtrsim \log^2(1/\delta)$,

\begin{equation*}
    \PP{ \maxEig{\sum_{t=0}^{T-1}\Delta_t\Delta_t^\top } \gtrsim
    1/\delta 
    \left(
    \sum_{t=1}^{T-1}\EE{t^{1/2}\norm{\Kh_t - K}^4 } +\log^2(1/\delta) +  T^{1/2} 
    \right)} \le 2\delta.
\end{equation*}

\end{lemma*}

\begin{proof}
Lemma \ref{lemma: favorable event} defines a high probability ``stable'' event, where when $T$ is large enough, the estimation errors are uniformly bounded by some constant.
See the proof at \cref{proof of lemma: favorable event}.

Lemma \ref{lemma: xt norm bound} establishes moment bounds in the ``stable'' event $E_\delta$.
\begin{lemma}
\label{lemma: xt norm bound}
\cref{alg:myAlg} applied to a system described by \cref{eq:system eq} under Assumption \ref{asm:InitialStableCondition} satisfies, for any $0 < \delta < 1/2$, $k \in \mathbb{N}$ and $T \gtrsim \log^2(1/\delta)$,
\begin{equation}
    \label{eq: improved Ext bound}
    \EE{\norm{x_{t}}^k1_{E_{\delta}}} \lesssim 1.
\end{equation}
\end{lemma}
See the proof at \cref{subsection: proof of lemma: xt norm bound}.

Notice that
\begin{equation*}
    \norm{\Delta_t} = \norm{(\Kh_t - K)x_t + \eta_t} \le \norm{\Kh_t - K}\norm{x_t} + \norm{\eta_t}.
\end{equation*}
Then, with probability $1-\delta$: 
\begin{align*}
    &\maxEig{\sum_{t=0}^{T-1}\Delta_t\Delta_t^\top 1_{E_{\delta}}}  \\
    & \quad \le \Trr{\sum_{t=0}^{T-1}\Delta_t\Delta_t^\top 1_{E_{\delta}}} \\
    & \quad \text{(by Markov inequality, this holds with probability } 1-\delta) \\
    & \quad \le 1/\delta \EE{\Trr{\sum_{t=0}^{T-1}\Delta_t\Delta_t^\top 1_{E_{\delta}}}} \\
    & \quad = 1/\delta \EE{\sum_{t=0}^{T-1}\norm{\Delta_t}^2 
    1_{E_{\delta}}}  \\
    & \quad \le 1/\delta 
    \sum_{t=0}^{T-1}\EE{\left(\norm{\Kh_t - K}\norm{x_t} + \norm{\eta_t}\right)^2 
    1_{E_{\delta}}}  \\
    & \quad \lesssim 1/\delta 
    \sum_{t=0}^{T-1}\EE{\left(\norm{\Kh_t - K}^2\norm{x_t}^2 + \norm{\eta_t}^2\right) 
    1_{E_{\delta}}}  \\
    & \quad \text{(Bound $t=0$ separately by constant 1 because $t^{-1/2}$ is not well defined)} \\
    & \quad \lesssim 1/\delta 
    \left(
    \sum_{t=1}^{T-1}\left(\EE{\left(t^{1/2}\norm{\Kh_t - K}^4 +  t^{-1/2}\norm{x_t}^4 \right) 
    1_{E_{\delta}}} + t^{-1/2} \right) + 1 
    \right) \\
    & \quad \lesssim 1/\delta 
    \left(
    \sum_{t=1}^{T-1}\EE{t^{1/2}\norm{\Kh_t - K}^4 +  t^{-1/2}\norm{x_t}^4 1_{E_{\delta}}
    }  + T^{1/2} 
    \right).
\end{align*}
By Lemma \ref{lemma: xt norm bound}, for $t \gtrsim \log^2(1/\delta)$, $\EE{\norm{x_t}^41_{E_{\delta}}} \lesssim 1$.
Thus
\begin{align}
\label{eq: deltat eigenvalue bound}
    &\maxEig{\sum_{t=0}^{T-1}\Delta_t\Delta_t^\top 1_{E_{\delta}}}  \nonumber \\
    & \quad \lesssim 1/\delta 
    \left(
    \sum_{t=1}^{T-1}\EE{t^{1/2}\norm{\Kh_t - K}^4} + 
    \sum_{t=1}^{C\log^2(1/\delta)}\EE{ t^{-1/2}\norm{x_t}^4} +
    \sum_{t = C\log^2(1/\delta)}^{T-1}t^{-1/2}  + 
    T^{1/2} 
    \right) \nonumber \\
    & \quad \text{(By \cref{eq: norm xt 4 bound},  $\E{\norm{x_t}^4} \lesssim \log^4(t)$)} \nonumber \\
    & \quad \lesssim 1/\delta 
    \left(
    \sum_{t=1}^{T-1}\EE{t^{1/2}\norm{\Kh_t - K}^4} + 
    \sum_{t=1}^{C\log^2(1/\delta)}t^{-1/2}\log^4(C\log^2(1/\delta)) +
    T^{1/2}  + 
    T^{1/2} 
    \right) \nonumber \\
    & \quad \lesssim 1/\delta 
    \left(
    \sum_{t=1}^{T-1}\EE{t^{1/2}\norm{\Kh_t - K}^4} + 
    (\log^2(1/\delta))^{1/2} * \left(2\log\log(1/\delta) + \log(C))\right)^4 +
    T^{1/2}  
    \right) \nonumber \\
    & \quad \lesssim 1/\delta 
    \left(
    \sum_{t=1}^{T-1}\EE{t^{1/2}\norm{\Kh_t - K}^4} + 
    \log^2(1/\delta) +
    T^{1/2}  
    \right). 
\end{align}

Finally we finish the proof by removing $1_{E_{\delta}}$ on the left hand side and changing the probability from $1-\delta$ to $1-2\delta$.
\end{proof}

\subsection{Proof of Lemma \ref{lemma: dtheta recursive}}
\label{subsection: proof of lemma: dtheta recursive}
\begin{lemma*}
\cref{alg:myAlg} applied to a system described by \cref{eq:system eq} under Assumption \ref{asm:InitialStableCondition} satisfies, for any $0 < \delta < 1/2$ and $T \gtrsim \log^3(1/\delta)$, 
\begin{align}
\label{eq: dtheta recursive eq*}
\begin{split}
    &\P\Bigg[
    T^{1/2}\lnorm{\hat\Theta_T-\Theta}^2 \gtrsim 
    \logg{T^{-1/2}\left(
    \sum_{t=1}^{T-1}\EE{t^{1/2}\norm{\Kh_t - K}^4} 
    \right) + 1 } + 
    \log\left(\frac{1}{\delta}\right)
    \Bigg] \le  3\delta.   
\end{split}
\end{align}
\end{lemma*}
\begin{proof}

Now we can use the refined upper and lower bound in Lemma \ref{lemma: better upper and lower bounds of GT}.
\begin{equation*}
    T\Gamlow \eqsim \begin{bmatrix}
            I_n \\
            K
        \end{bmatrix} T 
        \begin{bmatrix}
            I_n \\
            K
        \end{bmatrix}^\top + 
        \begin{bmatrix}
-K^\top \\
I_d
\end{bmatrix} T^{1/2}
\begin{bmatrix}
-K^\top \\
I_d
\end{bmatrix}^\top
\end{equation*}
and 
\begin{equation*}
    T\Gambar \eqsim 
    \frac1\delta 
    \begin{bmatrix}
I_n \\
K
\end{bmatrix} T 
\begin{bmatrix}
I_n \\
K
\end{bmatrix}^\top + 
\begin{bmatrix}
-K^\top \\
I_d
\end{bmatrix} \maxEig{ \sum_{t=0}^{T-1}\Delta_t\Delta_t^\top}
\begin{bmatrix}
-K^\top \\
I_d
\end{bmatrix}^\top
.
\end{equation*}
Then 
\begin{equation*}
    T\lambda_{\min}(\Gamlow) \eqsim T^{1/2}.
\end{equation*}
Also 
\begin{equation*}
    \det(T\Gamlow) = 
    \det\left(
        \begin{bmatrix}
            I_n  & -K^T\\
            K & I_d
        \end{bmatrix} 
        \begin{bmatrix}
        T  I_n  & 0\\
        0 & T^{1/2}I_d
        \end{bmatrix}
        \begin{bmatrix}
            I_n  & -K^T\\
            K & I_d
        \end{bmatrix}^\top
    \right) \eqsim T^{n} \cdot \left(T^{1/2}\right)^{d}.
\end{equation*}
and similarly
\begin{equation*}
    \det(T\Gambar) = 
    \det\left(
        \begin{bmatrix}
            I_n  & -K^T\\
            K & I_d
        \end{bmatrix} 
        \begin{bmatrix}
        T  I_n  & 0\\
        0 & \maxEig{ \sum_{t=0}^{T-1}\Delta_t\Delta_t^\top}I_d
        \end{bmatrix}
        \begin{bmatrix}
            I_n  & -K^T\\
            K & I_d
        \end{bmatrix}^\top
    \right) \eqsim T^{n} \cdot \maxEig{\sum_{t=0}^{T-1}\Delta_t\Delta_t^\top}^{d}.
\end{equation*}
With this particular upper and lower bound choices $\Gamlow$ and $\Gambar$, \cref{eq: simchowitz conclusion}
can be written as 
\begin{align*}
\P\left[\lnorm{\hat\Theta_T-\Theta} \gtrsim 
\sqrt{\frac{\log\left(\maxEig{ \sum_{t=0}^{T-1}\Delta_t\Delta_t^\top} T^{-1/2} \right) + \log\left(\frac{1}{\delta}\right)}{T^{1/2}}} \right] \le  \delta.
\end{align*}
Then we apply Lemma \ref{lemma: lambda max sum of delta deltaT bound}:
\begin{align*}
\begin{split}
    &\P\Bigg[
    T^{1/2}\lnorm{\hat\Theta_T-\Theta}^2 \gtrsim 
    \logg{T^{-1/2}\left(
    \sum_{t=1}^{T-1}\EE{t^{1/2}\norm{\Kh_t - K}^4} + 
    \log^2(1/\delta)
    \right) + 1 } + 
    \log\left(\frac{1}{\delta}\right)
    \Bigg] \le  3\delta.   
\end{split}
\end{align*}
Notice that
\begin{align*}
    & \logg{T^{-1/2}\left(
    \sum_{t=1}^{T-1}\EE{t^{1/2}\norm{\Kh_t - K}^4} + 
    \log^2(1/\delta)
    \right) + 1 } + 
    \log\left(\frac{1}{\delta}\right) 
    \\
    & \quad \le
    \logg{T^{-1/2}\left(
    \sum_{t=1}^{T-1}\EE{t^{1/2}\norm{\Kh_t - K}^4} 
    \right) + \log^2(1/\delta) + 1 } + 
    \log\left(\frac{1}{\delta}\right) 
    \\
    & \quad \text{(because $2ab \ge a + b$ when both $a \ge 1$ and $b \ge 1$)} \\
    & \quad \le
    \logg{
    2
    \left(
        T^{-1/2}
        \left(
            \sum_{t=1}^{T-1}\EE{t^{1/2}\norm{\Kh_t - K}^4}  
        \right) +  1 
    \right) * \log^2(1/\delta)} + 
    \log\left(\frac{1}{\delta}\right) 
    \\
    & \quad \lesssim
    \logg{
        T^{-1/2}
        \left(
            \sum_{t=1}^{T-1}\EE{t^{1/2}\norm{\Kh_t - K}^4}  
        \right) +  1 
    } + 
    2\logg{\log(1/\delta)} + \log(2) + 
    \log\left(\frac{1}{\delta}\right) 
    \\
    & \quad \lesssim
    \logg{
        T^{-1/2}
        \left(
            \sum_{t=1}^{T-1}\EE{t^{1/2}\norm{\Kh_t - K}^4}  
        \right) +  1 
    } + 
    \log\left(\frac{1}{\delta}\right).
\end{align*}
Finally we have
\begin{align*}
\begin{split}
    &\P\Bigg[
    T^{1/2}\lnorm{\hat\Theta_T-\Theta}^2 \gtrsim \logg{
        T^{-1/2}
        \left(
            \sum_{t=1}^{T-1}\EE{t^{1/2}\norm{\Kh_t - K}^4}  
        \right) +  1 
    } + 
    \log\left(\frac{1}{\delta}\right)
    \Bigg] \le  3\delta.   
\end{split}
\end{align*}

\end{proof}

\subsection{Proof of \cref{theorem: theta bound without log}}
\label{subsection: lemma: theta bound without log}

\begin{theorem*}
\cref{alg:myAlg} applied to a system described by \cref{eq:system eq} under Assumption \ref{asm:InitialStableCondition} satisfies
\begin{equation*}
    \lnorm{\hat\Theta_T-\Theta} = O_p(T^{-1/4})\text{ and }\lnorm{\Kh_T-K} = O_p(T^{-1/4}).
\end{equation*}
\end{theorem*}

\begin{proof}

By Proposition 4 of \citet{simchowitz2020naive}, 
\begin{equation}
\label{eq: delta K controlled by delta theta*}
    \norm{\Kh_T - K} \lesssim \norm{\hat\Theta_T - \Theta}.
\end{equation}
as long as $\norm{\hat\Theta_T - \Theta} \le \epsilon_0$, where $\epsilon_0$ is some fixed constant determined by the system parameters (this is the same $\epsilon_0$ as in \cref{lemma: favorable event}). We want to focus on cases where $\norm{\hat\Theta_T - \Theta} \le \epsilon_0$ to transfer $T^{1/2}\lnorm{\hat\Theta_T-\Theta}^2$ to $T^{1/2}\lnorm{\Kh_T-K}^2$ so that Lemma \ref{lemma: dtheta recursive} has only estimation error of $K$.

We can estimate $\EE{T\lnorm{\hat\Theta_T-\Theta}^4 1_{\norm{\hat\Theta_T-\Theta} \le \epsilon_0}}$ by calculating the integral using the tail bound from Lemma \ref{lemma: dtheta recursive} as long as $T \gtrsim \log^3(1/\delta)$. The further tails can be bounded by the constant $\epsilon_0$. 
Add an extra $1_{\norm{\hat\Theta_T-\Theta} \le \epsilon_0}$ on the left hand side of \cref{eq: dtheta recursive eq*} inside the probability. When $0 < \delta < 1/2$ and $T \gtrsim \log^3(1/\delta)$,
\begin{align*}
\begin{split}
    &\P\Bigg[
    T^{1/2}\lnorm{\hat\Theta_T-\Theta}^2 1_{\norm{\hat\Theta_T-\Theta} \le \epsilon_0} \gtrsim 
    \logg{
    T^{-1/2}\left(
    \sum_{t=1}^{T-1}\EE{t^{1/2}\norm{\Kh_t - K}^4} 
    \right) + 1 
    } + 
    \log\left(\frac{1}{\delta}\right)
    \Bigg] \le  3\delta.   
\end{split}
\end{align*}

Denote $a_T = T^{1/2}\lnorm{\hat\Theta_T-\Theta}^2 1_{\norm{\hat\Theta_T-\Theta} \le \epsilon_0}$, $C_{T} := T^{-1/2}\left(
    \sum_{t=1}^{T-1}\EE{t^{1/2}\norm{\Kh_t - K}^4} 
    \right) + 1 $.  When $T \gtrsim \log^3(1/\delta)$, which is $\delta \ge e^{-(cT)^{1/3}}$, we have
\begin{equation}
\label{eq: summarize probability}
    \PP{a_T \ge C\left(\log{C_{T}} + \logg{\frac{1}{\delta}}\right)} \le 3\delta.
\end{equation}
Here $c$ and $C$ are two fixed constants which do not depend on $\delta$ and $T$.
Denote the tail bound of $a_T$ corresponding to probability $\delta = e^{-(cT)^{1/3}}$ as $U_T := C\left(\logg{C_{T}} + (cT)^{1/3}\right)$. When $a_T > U_T$, we bound it by the bound $a_T \le T^{1/2} \cdot \epsilon_0^2$.

\begin{align*}
\E a_{T}^2 
    &\le \int_{s = 0}^{U_{T}} s\PP{a_{T}^2 = s} ds 
    + 3e^{-(cT)^{1/3}} T \epsilon_0^4
    \\
    &\le -\int_{s=0}^{U_{T}} s \; d\PP{a_{T}^2 > s} 
    + 3e^{-(cT)^{1/3}} T \epsilon_0^4 
    \\
    &\le -U_{T}\PP{a_{T}^2 > U_{T}} + \int_{s=0}^{U_{T}} \PP{a_{T}^2 > s} \; ds 
    + 3e^{-(cT)^{1/3}} T \epsilon_0^4 
    \\
    &\quad \text{(because } a_{T} > 0 \text{)} \\
    &\le \int_{s=0}^{U_T} \PP{a_{T} > \sqrt{s}} ds 
    + 3e^{-(cT)^{1/3}} T \epsilon_0^4 
    \\
    &\le \int_{s = (C\log(2C_{T}) )^2}^{U_T} \PP{a_{T} > \sqrt{s}} ds + (C\log(2C_{T}) )^2 + 3e^{-(cT)^{1/3}} T \epsilon_0^4 
\end{align*}
We choose $s = (C\log(2C_{T}) )^2$ because we only know the bound of $\PP{a_{T} > \sqrt{s}}$ up to $\delta = 1/2$ with the restriction $\delta < 1/2$. We know that $s = (C\log(2C_{T}) )^2$ corresponds to $\delta = 1/2$ by \cref{eq: summarize probability}.

We also need to express probability $\delta$ in terms of the tail value $s$ as $\delta(s)$.
Solve the equation
\begin{align*}
    &\sqrt{s} = C\left(\log{C_{T}} + \logg{\frac{1}{\delta(s)}}\right) \\
    \Longrightarrow \quad &
    e^{\sqrt{s}/C} = C_{T} \cdot \frac{1}{\delta(s)} \\
    \Longrightarrow \quad &
    \delta(s)  = e^{-\sqrt{s}/C} C_{T}.
\end{align*}
Thus
\begin{align*}
    \E a_{T}^2 
    &\le \int_{s=(C\log(2C_{T}) )^2}^{U_T} \PP{a_{T} > \sqrt{s}} ds + (C\log(2C_{T}) )^2 + 3e^{-(cT)^{1/3}} T \epsilon_0^4\\
    &\le \int_{s=(C\log(2C_{T}) )^2}^{U_T} 3\delta(s) \;ds + (C\log(2C_{T}) )^2 + 3e^{-(cT)^{1/3}} T \epsilon_0^4\\
    &\le \int_{s=(C\log(2C_{T}) )^2}^{\infty} 3e^{-\sqrt{s}/C} C_{T} ds + (C\log(2C_{T}) )^2 + 3e^{-(cT)^{1/3}} T \epsilon_0^4\\
    & \text{(By integral calculation)}\\
    &= 3C(C\log(2C_{T}) + C) + (C\log(2C_{T}) )^2 + 3e^{-(cT)^{1/3}} T \epsilon_0^4\\
    &\eqsim \log(2C_{T}) + (\log(2C_{T}) )^2 +  e^{-(cT)^{1/3}} T \epsilon_0^4 \\
    &= (\log(2C_{T}) + 1/2)^2 - 1/4  + e^{-(cT)^{1/3}} T \epsilon_0^4 \\
    &\lesssim (\log(C_{T}) +1)^2 + e^{-(cT)^{1/3}} T \epsilon_0^4.
\end{align*}
As a result, when $T \ge T_0$ ($T_0$ is a large enough constant so that $e^{-cT^{1/3}}T\epsilon_0^4 \le 1$):
\begin{align*}
    \EE{T\lnorm{\hat\Theta_T-\Theta}^4 1_{\norm{\hat\Theta_T-\Theta} \le \epsilon_0}} &\lesssim \left(
        \log\left(T^{-1/2}\left(\sum_{t=1}^{T-1}t^{-1/2} \EE{t\lnorm{\Kh_t-K}^4} \right) \right)  + 1
    \right)^2 + e^{-cT^{1/3}}T\epsilon_0^4 \\
    &\lesssim \left(
        \log\left(T^{-1/2}\left(\sum_{t=1}^{T-1}t^{-1/2} \EE{t\lnorm{\Kh_t-K}^4} \right) \right)  + 1
    \right)^2 + 1.
\end{align*}
On the right hand side, consider the maximum of $\EE{t\lnorm{\Kh_t-K}^4}$ from $T_0$ to $T_{\max} \ge T$,
\begin{align*}
    &\EE{T\lnorm{\hat\Theta_T-\Theta}^4
    1_{\norm{\hat\Theta_T-\Theta} \le \epsilon_0}
    } \\
    &\quad \lesssim \left(
        \log\left(T^{-1/2}\left(\sum_{t=1}^{T_0}t^{-1/2} \EE{t\lnorm{\Kh_t-K}^4}
        + \sum_{t=T_0}^{T-1}t^{-1/2} \max_{T_0 \le s \le T_{\max}} \EE{s\lnorm{\Kh_s-K}^4}
        \right) \right)  + 1
    \right)^2 + 1 \\
    & \quad 
    \text{(\cref{alg:myAlg} restricted $\norm{\Kh_t} \le C_K$)
    }
    \\
    & \quad \lesssim \left(
        \log\left(T^{-1/2}\left(\sum_{t=1}^{T_0}t^{-1/2} (C_K + \norm{K})^4
        \right)+ \max_{T_0 \le s \le T_{\max}} \EE{s\lnorm{\Kh_s-K}^4} \right)  + 1
    \right)^2 + 1  \\
    &\quad \lesssim \left(
        \log\left(T^{-1/2}T_0^{1/2}\cdot 1+ \max_{T_0 \le s \le T_{\max}} \EE{s\lnorm{\Kh_s-K}^4} \right)  + 1
    \right)^2 + 1 \\
    &\quad \lesssim 
    \left(
        \log\left(1+ \max_{T_0 \le s \le T_{\max}} \EE{s\lnorm{\Kh_s-K}^4} \right)  + 1
    \right)^2 + 1.
\end{align*}
By \cref{eq: delta K controlled by delta theta*}, we can transfer $\lnorm{\hat\Theta_T-\Theta}$ on the right hand side to $\lnorm{\Kh_T-K}$ as long as $\lnorm{\hat\Theta_T-\Theta} \le \epsilon_0$.
By Lemma \ref{lemma: delta theta bound}, the upper bound for the probability $\delta$ that $\lnorm{\hat\Theta_T-\Theta} \le \epsilon_0$ does not hold can be solved from 
\begin{equation*}
    T^{-1/4}\sqrt{
    \left(
    \log T + \log\left(\frac{1}{\delta}\right)
    \right)} = \epsilon_0,
\end{equation*}
which gives
\begin{equation*}
    \delta = T e^{-\epsilon_0^2T^{1/2}}.
\end{equation*}
As a result, when $T \ge T_0$:
\begin{align*}
    &\EE{T\lnorm{\hat\Theta_T-\Theta}^4
     1_{\norm{\hat\Theta_T-\Theta} \le \epsilon_0}
    }  \\
    & \quad \gtrsim 
    \EE{T\lnorm{\Kh_T-K}^4
     1_{\norm{\hat\Theta_T-\Theta} \le \epsilon_0}
    }  \\
    & \quad \ge \EE{T\lnorm{\Kh_T-K}^4} - T e^{-\epsilon_0^2T^{1/2}}T(C_K + \norm{K})^4 \\
    & \quad \text{(Choose $T_0$ such that for any $T \ge T_0$, $T e^{-\epsilon_0^2T^{1/2}}T(C_K + \norm{K})^4 \le 1$)} \\
    & \quad \ge \EE{T\lnorm{\Kh_T-K}^4} - 1.
\end{align*}
Now we have, for any $T_0 \le T \le T_{\max}$
\begin{equation*}
\EE{T\lnorm{\Kh_T-K}^4} \lesssim 
        \left(
        \log\left(1+ \max_{T_0 \le s \le T_{\max}} \EE{s\lnorm{\Kh_s-K}^4} \right)  + 1
    \right)^2 + 1.
\end{equation*}
Take maximum across $T_0$ to $T_{\max}$ on the left hand side:
\begin{equation*}
    \max_{T_0 \le s \le T_{\max}} \EE{s\lnorm{\Kh_s-K}^4}
    \lesssim \left(
        \log\left(1+ \max_{T_0 \le s \le T_{\max}} \EE{s\lnorm{\Kh_s-K}^4} \right)  + 1
    \right)^2 + 1.
\end{equation*}
Thus 
\begin{equation*}
    \max_{T_0 \le s \le T_{\max}} \EE{s\lnorm{\Kh_s-K}^4} \lesssim 1.
\end{equation*}
The hidden constant is only related with $T_0$. The same inequality hold for any $T_{\max}$. As a result,
\begin{equation}
\label{eq:E K estimation err ^4 bound}
    \max_{s \ge T_0} \EE{s\lnorm{\Kh_s-K}^4}  
    \lesssim 1.
\end{equation}
Plug this back to \cref{eq: dtheta recursive eq*}. When $T \gtrsim \log^3(1/\delta)$,
\begin{equation*}
    \P\left[
    T^{1/2}\lnorm{\hat\Theta_T-\Theta}^2 \gtrsim 
    \logg{T^{-1/2}\left(\sum_{t=1}^{T_0}t^{-1/2} \EE{t\lnorm{\Kh_t-K}^4}
    \right)+ 1} + \log\left(\frac{1}{\delta} \right) 
    \right] \le  3\delta.
\end{equation*}
Because \cref{alg:myAlg} restricted that $\norm{\Kh_t} \le C_K$, the initial $T_0$ items is of negligible order. The above equation can be simplified as
\begin{equation*}
    \P\left[
    T^{1/2}\lnorm{\hat\Theta_T-\Theta}^2 \gtrsim 
    \log\left(\frac{1}{\delta} \right) 
    \right] \le  3\delta.
\end{equation*}
Finally,
\begin{equation*}
    \lnorm{\hat\Theta_T-\Theta} = O_p(T^{-1/4}).
\end{equation*}
$\lnorm{\Kh_T-K} = O_p(T^{-1/4})$ is a direct corollary from \cref{eq: delta K controlled by delta theta*}. 
\end{proof}

\subsection{Proof of Lemma \ref{lemma: large eigenvalue part}}
\label{subsection: proof of lemma: large eigenvalue part}
\begin{lemma*}
\cref{alg:myAlg} applied to a system described by \cref{eq:system eq} under Assumption \ref{asm:InitialStableCondition} satisfies, for any $0 < \delta < 1/2$ and $T \gtrsim \log^3(1/\delta)$, with probability at least $1-\delta$:
\begin{equation}
\label{eq: upper lower bound of the main part in gram matrix**}
\begin{bmatrix}
I \\
K
\end{bmatrix} T 
\begin{bmatrix}
I \\
K
\end{bmatrix}^\top
\preceq
\begin{bmatrix}
I \\
K
\end{bmatrix}
\sum_{t=0}^{T-1} 
x_t x_t^\top
\begin{bmatrix}
I \\
K
\end{bmatrix}^\top \preceq 
1/\delta
\begin{bmatrix}
I \\
K
\end{bmatrix} T 
\begin{bmatrix}
I \\
K
\end{bmatrix}^\top.
\end{equation} 

\end{lemma*}

\paragraph{High probability upper bound}
For $t \lesssim \log^2(1/\delta)$, we can use the bound from Eq.~(104) in \citet{wang2020exact}: $\E \norm{x_t}^2 \lesssim \log^2(t)$.  For $t \lesssim \log^2(1/\delta)$ part, the total effect is bounded by $\log^2(1/\delta)\log^2\left(\log^2(1/\delta)\right)
\lesssim  \log^3(1/\delta)$.
For $t \gtrsim \log^2(1/\delta)$ , we can use Lemma \ref{lemma: xt norm bound}: $\EE{\norm{x_{t}}^21_{E_{\delta}}} \lesssim 1.$
We then combine the bounds for $t \lesssim \log^2(1/\delta) $ and $t \gtrsim \log^2(1/\delta)$ to get
\begin{align*}
    &\EE{\sum_{t=0}^{T-1}x_{t}x_t^\top 1_{E_{\delta}}} \\
    & \quad = 
\sum_{t=0}^{T-1}\EE{x_{t}x_t^\top1_{E_{\delta}}} \\
& \quad \preceq 
\sum_{t=0}^{T-1}\EE{\norm{x_t}^21_{E_{\delta}}}I_{n} \\
& \quad \preceq 
\left(T  + \log^3(1/\delta)\right)I_{n}.
\end{align*}
In order to make the formula neat, require $T \gtrsim \log^3(1/\delta)$, which guarantees the simplified formula
\begin{align}
\label{eq: Expectation Gram matrix bounded by T}
    \EE{\sum_{t=0}^{T-1}x_{t}x_t^\top 1_{E_{\delta}}}  \precsim 
T \cdot  I_n.
\end{align}

By \cref{eq: matrix M upper bound} we have:
\begin{equation*}
    \PP{\sum_{t=0}^{T-1}x_{t}x_t^\top 1_{E_{\delta}} \preceq \frac{d}{\delta}\E[\sum_{t=0}^{T-1}x_{t}x_t^\top 1_{E_{\delta}}]} \ge 1-\delta.
\end{equation*}
Further combine this with \cref{eq: Expectation Gram matrix bounded by T}:
\begin{equation*}
    \PP{\sum_{t=0}^{T-1}x_{t}x_t^\top 1_{E_{\delta}} \preceq \frac{C}{\delta} TI_{n} } \ge 1- \delta.
\end{equation*}
Also we can remove the $1_{E_{\delta}}$ part by subtracting another $\delta$ on the right:
\begin{equation*}
    \PP{\sum_{t=0}^{T-1}x_{t}x_t^\top \precsim \frac{1}{\delta} TI_{n} } \ge 1-2\delta
\end{equation*}
or just hide the constant $2$ by $\delta \to \delta/2$. Now we can conclude that when $T \gtrsim \log^3(1/\delta)$:
\begin{equation}
\label{eq: matrix upper bound}
    \PP{\sum_{t=0}^{T-1}x_{t}x_t^\top \precsim \frac{1}{\delta} TI_{n} } \ge 1-\delta.
\end{equation}
or, with probability $1-\delta$, 
\begin{equation*}
    \begin{bmatrix}
I \\
K
\end{bmatrix}
\sum_{t=0}^{T-1} 
x_t x_t^\top
\begin{bmatrix}
I \\
K
\end{bmatrix}^\top \precsim 
1/\delta
\begin{bmatrix}
I \\
K
\end{bmatrix} T
\begin{bmatrix}
I \\
K
\end{bmatrix}^\top.
\end{equation*}

\paragraph{High probability lower bound}
Next we want to show a high probability lower bound of this term, which is a component in $G_T$. It is sufficient to prove some BMSB condition because when BMSB condition and high probability upper bounds both hold, the lower bound is also guaranteed (we will illustrate this soon).

Following \cref{def:bmsb}, in order to show the process $\{x_t\}_{t \ge 1}$ satisfies the $(k, \Gamlow, p) = (1, \sigma_\varepsilon^2 I_n, \frac3{10})$-BMSB condition, we only need to prove that for any $w\in \calS^{n-1}$, $\P( |\langle w, x_{j+1}\rangle | \ge \sqrt{w^\top \sigma_\varepsilon^2 I_n w} | \calF_{j}) \ge p \as$
Let $\calF_t$ be the filtration on all history before time $t$ (including $x_t$ and $u_t$), we know that
\begin{equation*}
    x_{t+1} | \calF_t \sim \calN\left(Ax_t + Bu_t, \sigma_\varepsilon^2 I_n\right).
\end{equation*}
We also know the distribution of its inner product with any constant vector $w$:
\begin{equation*}
    \langle w, x_{t+1} \rangle | \calF_t \sim \calN\left(\langle w, Ax_t + Bu_t \rangle, w^\top \sigma_\varepsilon^2 I_n w\right). 
\end{equation*}
We want to lower bound the probability that the absolute value of this inner product (which follows a normal distribution) is larger than its standard error, which is always lower bounded by the case where the normal distribution is centered at zero. More specifically, 
\begin{align*}
    &\P\left( |\langle w, x_{j+1}\rangle | \ge \sqrt{w^\top \sigma_\varepsilon^2 I_n w} | \calF_{j}
    \right) \\
   & \quad \ge 
   \P\left( \left|\calN\left(0, w^\top \sigma_\varepsilon^2 I_n w\right) \right| \ge \sqrt{w^\top \sigma_\varepsilon^2 I_n w}
   \right) \\
   & \quad \ge 3/10.
\end{align*}
The last equation is simply a numerical property of the normal distribution. Now we have proved that the process $\{x_t\}_{t \ge 1}$ follows $(k, \Gamlow, p) = (1, \sigma_\varepsilon^2 I_n, \frac3{10})$-BMSB condition.

The BMSB-condition is useful in deriving high probability lower bounds. Specifically, 
assume $X = (x_0, x_1, \dots, x_{T-1})$, then if the Gram matrix $\sum_{t=0}^{T-1}x_{t}x_t^\top$ has a high probability upper bound, then the BMSB-condition can guarantee a high probability lower bound. In the last equation from section D.1 in \citet{simchowitz2018learning}, it is shown that if $\{x_t\}_{t \ge 1}$ satisfies the $(k, \Gamlow, p)$-BMSB condition, then
\begin{equation}
\label{eq: lower bound gram matrix by BMSB origin}
    \PP{
        \mathset{
            \sum_{t=0}^{T-1}x_{t}x_t^\top \nsucceq \frac{k\floor{T/k}p^2\Gamlow}{16}
        }
        \cap
        \mathset{
            \sum_{t=0}^{T-1}x_{t}x_t^\top \preceq T\bar\Gamma
        }
    }
    \le 
    \expebrace{-\frac{Tp^2}{10k} + 2d\log(10/p) + \log\det(\bar\Gamma\Gamlow^{-1})}.
\end{equation}
Here $\Gambar$ comes from the assumption in \cref{eq: simchowitz conclusion} which says $\P[\sum_{t = 0}^{T-1}  z_t z_t^\top \preceq T\Gambar] \ge 1-\delta$. the upper bound $T\Gambar$ is guaranteed by \cref{eq: matrix upper bound} with $\Gambar \simeq 1/\delta
I_n$, and the lower bound is just shown to be $\Gamlow = \sigma_\varepsilon^2 I_n$ with $k=1$ and $p=3/10$. Put these representations into the previous equation 
\begin{equation*}
    \PP{
        \mathset{
            \sum_{t=0}^{T-1}x_{t}x_t^\top \nsucceq \frac{\floor{T}\left(\frac{3}{10}\right)^2\Gamlow}{16}
        }
        \cap
        \mathset{
            \sum_{t=0}^{T-1}x_{t}x_t^\top \preceq T\bar\Gamma
        }
    }
    \le 
    \expebrace{-\frac{T\left(\frac{3}{10}\right)^2}{10} + C + d\log(1/\delta)}.
\end{equation*}
Here $C$ is some constant independent of $\delta$ and $T$.
To make the right hand side smaller than $\delta$, the condition is $\expebrace{-\frac{9T}{1000} + C + d\log(1/\delta)} < \delta$, which means
\begin{equation*}
    \frac{9T}{1000} - C - d\log(1/\delta) > \log(1/\delta),
\end{equation*}
which is just $T \gtrsim \logg{1/\delta}$. With this condition, we have 
\begin{equation*}
    \PP{
        \mathset{
            \sum_{t=0}^{T-1}x_{t}x_t^\top \nsucceq \frac{9\floor{T}\Gamlow}{1600}
        }
        \cap
        \mathset{
            \sum_{t=0}^{T-1}x_{t}x_t^\top \preceq T\bar\Gamma
        }
    }
    \le 
    \delta.
\end{equation*}
which is
\begin{equation*}
    \PP{
        \mathset{
            \sum_{t=0}^{T-1}x_{t}x_t^\top \nsucceq \frac{9\floor{T} \sigma_\varepsilon^2 I_n }{1600}
        }
        \cap
        \mathset{
            \sum_{t=0}^{T-1}x_{t}x_t^\top \precsim \frac{1}{\delta}TI_n
        }
    }
    \le 
    \delta.
\end{equation*}
We can exclude the later event and change the probability on the right hand side to $\delta + \delta = 2\delta$. When $T \gtrsim \log(1/\delta)$, 
\begin{equation*}
    \PP{
        \mathset{
            \sum_{t=0}^{T-1}x_{t}x_t^\top \nsucceq \frac{9\floor{T} \sigma_\varepsilon^2 I_n }{1600}
        }
    }
    \le 
    2\delta.
\end{equation*}
We can change $2\delta$ to $\delta$, and the constraint is still $T \gtrsim \log(1/\delta)$.
\begin{equation}
\label{eq: lower bound gram matrix by BMSB}
    \PP{
        \mathset{
            \sum_{t=0}^{T-1}x_{t}x_t^\top \nsucceq \frac{9\floor{T} \sigma_\varepsilon^2 I_n }{1600}
        }
    }
    \le 
    \delta.
\end{equation}

Now with probability $1-2\delta$ (one $\delta$ from upper bound \cref{eq: matrix upper bound}, 
another $\delta$ from lower bound \cref{eq: lower bound gram matrix by BMSB}) we have both upper and lower bound of 
\begin{equation*}
    T I_n \precsim \sum_{t=0}^{T-1} x_t x_t^\top
    \precsim  \frac{1}{\delta}TI_n,
\end{equation*}
and
\begin{equation*}
\begin{bmatrix}
I \\
K
\end{bmatrix} T 
\begin{bmatrix}
I \\
K
\end{bmatrix}^\top
\precsim
\begin{bmatrix}
I \\
K
\end{bmatrix}
\sum_{t=0}^{T-1} 
x_t x_t^\top
\begin{bmatrix}
I \\
K
\end{bmatrix}^\top \precsim
1/\delta
\begin{bmatrix}
I \\
K
\end{bmatrix} T 
\begin{bmatrix}
I \\
K
\end{bmatrix}^\top.
\end{equation*} 
when $T \gtrsim \log^3(1/\delta)$. We can replace $\delta$ with $\delta/2$ so that $1-2\delta$ becomes $1-\delta$. 
$\blacksquare$

\subsection{Proof of Lemma \ref{lemma: lower bound of GT}}
\label{subsection: Proof of lemma: lower bound of GT}
\begin{lemma*}
\cref{alg:myAlg} applied to a system described by \cref{eq:system eq} under Assumption \ref{asm:InitialStableCondition} satisfies, when $0 < \delta < 1/2$, for any $T \gtrsim \log^3(1/\delta)$, with probability at least $1-\delta$,
\begin{equation*}
    G_T  \succsim 
        \begin{bmatrix}
            I_n \\
            K
        \end{bmatrix} T 
        \begin{bmatrix}
            I_n \\
            K
        \end{bmatrix}^\top + 
        T^{1/2}  I_{n+d}.
\end{equation*}
\end{lemma*}
By definition \cref{eq: GT definition},
$G_T-\begin{bmatrix}
I \\
K
\end{bmatrix}
\sum_{t=0}^{T-1} 
x_t x_t^\top
\begin{bmatrix}
I \\
K
\end{bmatrix}^\top\succeq 0$, thus by Lemma \ref{lemma: large eigenvalue part}
\begin{equation}
\label{eq: first lower bound of GT}
    G_T \succsim 
\begin{bmatrix}
I \\
K
\end{bmatrix} T 
\begin{bmatrix}
I \\
K
\end{bmatrix}^\top.
\end{equation}
This lower bound is growing linearly with $T$ but still is low rank, so we combine this with another lower bound which is full rank but grows sub-linearly with $T$.

By Lemma 34 of \citet{wang2020exact}, the process $(z_t)_{t = 0}^{T-1}$ satisfies the $(1, \sigma_\eta^2 T^{-1/2}I_{n+d}, \frac3{10})$-BMSB condition. Here $z_t = \begin{bmatrix}
    x_t \\
    u_t
\end{bmatrix}$. Now we only need an upper bound to gaurantee the lower bound using BMSB condition. Again by \cref{eq: matrix M upper bound}, we have
\begin{equation*}
    \PP{\sum_{t=0}^{T-1} z_tz_t^\top \npreceq \frac{n+d}{\delta}\E\left[\sum_{t=0}^{T-1} z_tz_t^\top\right]} \le \delta.
\end{equation*}
Also by Eq. (104) from \citet{wang2020exact}, $\E\norm{z_t}^2 \lesssim \log^2(t)$. Thus $\E\left[\sum_{t=0}^{T-1} z_tz_t^\top\right] \lesssim \log^2(T)T$, we have
\begin{equation*}
    \PP{\sum_{t=0}^{T-1} z_tz_t^\top \npreceq \frac{C}{\delta}\log^2(T)TI_{n+d}} \le \delta.
\end{equation*}

Now that we have an upper bound, similar to \cref{eq: lower bound gram matrix by BMSB}, we can get the BMSB implied lower bound with $\Gamlow \eqsim T^{-1/2}I_{n+d}$: when $T \gtrsim \log(1/\delta)$, with probability at least $1-\delta$, $G_T = \sum_{t=0}^{T-1} z_tz_t^\top \succsim  T^{1/2}I_{n+d}$. Combining this and \cref{eq: first lower bound of GT}, with probability at least $1-\delta$:
\begin{equation*}
    G_T + G_T \succsim 
        \begin{bmatrix}
            I_n \\
            K
        \end{bmatrix} T 
        \begin{bmatrix}
            I_n \\
            K
        \end{bmatrix}^\top + 
        T^{1/2}  I_{n+d}.
\end{equation*}
We derived the lower bound, and we hope that the upper bound can have a similar form.

\subsection{Proof of Lemma \ref{lemma: upper bound of $G_T$}}
\label{subsection: Proof of lemma: upper bound of $G_T$}

\begin{lemma*}[Upper bound of $G_T$]
\cref{alg:myAlg} applied to a system described by \cref{eq:system eq} under Assumption \ref{asm:InitialStableCondition} satisfies, when $0 < \delta < 1/2$, for any $T \gtrsim \log^3(1/\delta)$, with probability at least $1-\delta$:
\begin{equation}
\label{eq: G_T upper bound*}
    G_T \precsim \left(
    \frac1\delta 
    \begin{bmatrix}
I_n \\
K
\end{bmatrix} T
\begin{bmatrix}
I_n \\
K
\end{bmatrix}^\top + 
\begin{bmatrix}
-K^\top \\
I_d
\end{bmatrix} \maxEig{ \sum_{t=0}^{T-1}\Delta_t\Delta_t^\top}
\begin{bmatrix}
-K^\top \\
I_d
\end{bmatrix}^\top
\right).
\end{equation}
\end{lemma*}

\begin{proof}
Any vector $\xi \in \real^{n+d}$ can be represented as the summation of vectors $\xi_1 \in \real^{n+d}$ and $\xi_2 \in \real^{n+d}$ from orthogonal subspaces spanned by the columns of
$\begin{bmatrix}
I_n \\
K
\end{bmatrix}$ and $\begin{bmatrix}
-K^\top \\
I_d
\end{bmatrix}$. We only need to show that, for any $\xi_1 = \begin{bmatrix}
    I_n \\
    K
\end{bmatrix} \alpha_1$ and $\xi_2 = 
\begin{bmatrix}
    -K^\top \\
    I_d
\end{bmatrix} \alpha_2
$ (with $\alpha_1 \in \real^n$, $\alpha_2 \in \real^d$)
, for any $T \gtrsim \log^3(1/\delta)$, with probability at least $1-\delta$, we have
\begin{equation*}
    (\xi_1 + \xi_2)^\top G_T (\xi_1 + \xi_2) \precsim 
    (\xi_1 + \xi_2)^\top
    \left(
    \frac1\delta 
    \begin{bmatrix}
I_n \\
K
\end{bmatrix} T 
\begin{bmatrix}
I_n \\
K
\end{bmatrix}^\top + 
\begin{bmatrix}
-K^\top \\
I_d
\end{bmatrix} \maxEig{ \sum_{t=0}^{T-1}\Delta_t\Delta_t^\top}
\begin{bmatrix}
-K^\top \\
I_d
\end{bmatrix}^\top
\right)
(\xi_1 + \xi_2).
\end{equation*}
We then show this inequality by the following two inequalities:
\begin{enumerate}
    \item When $T \gtrsim \log^3(1/\delta)$, with probability at least $1-\delta$:
    \begin{align*}
    &(\xi_1 + \xi_2)^\top G_T (\xi_1 + \xi_2) \\
    &\quad \le 2 \xi_1^\top G_T \xi_1 + 2 \xi_2^\top G_T \xi_2 \\
    &\quad \text{(because $\xi_2$ is orthogonal to $\begin{bmatrix}
        I \\
        K
        \end{bmatrix}
        \sum_{t=0}^{T-1} 
        x_t x_t^\top
        \begin{bmatrix}
        I \\
        K
        \end{bmatrix}^\top$)} \\
    &\quad = 2 \xi_1^\top G_T \xi_1 + 
        2 \xi_2^\top 
        \sum_{t=0}^{T-1}
    \begin{bmatrix}
    0_n &  x_t \Delta_t^\top\\
    \Delta_tx_t^\top  & \Delta_t\Delta_t^\top + K x_t \Delta_t^\top + \Delta_tx_t^\top K^\top  \\
    \end{bmatrix}
        \xi_2 \\
    &\quad = 2 \xi_1^\top G_T \xi_1 + 
        2 \left(\begin{bmatrix}
    -K^\top \\
    I_d
    \end{bmatrix} \alpha_2\right)^\top 
        \sum_{t=0}^{T-1}
    \begin{bmatrix}
    0_n &  x_t \Delta_t^\top\\
    \Delta_tx_t^\top  & \Delta_t\Delta_t^\top + K x_t \Delta_t^\top + \Delta_tx_t^\top K^\top  \\
    \end{bmatrix}
        \left(\begin{bmatrix}
        -K^\top \\
        I_d
    \end{bmatrix} \alpha_2\right)
    \\
    &\quad = 2 \xi_1^\top G_T \xi_1 + 
        2 \left(\alpha_2\right)^\top 
    \left(
    \sum_{t=0}^{T-1}\Delta_t\Delta_t^\top 
    \right)
        \left(\alpha_2\right)
    \\
    &\quad \le 2 \xi_1^\top \sum_{t=0}^{T-1}z_{t}z_t^\top \xi_1 + 
        2\norm{\alpha_2}^2
        \maxEig{ \sum_{t=0}^{T-1}\Delta_t\Delta_t^\top}
    \\
    &\quad \text{(Similar to \cref{eq: matrix upper bound}, when $T \gtrsim \log^3(1/\delta)$, $ \PP{\sum_{t=0}^{T-1}z_{t}z_t^\top \precsim \frac{1}{\delta} TI_{n+d} } \ge 1-\delta.$)} \\
    &\quad \lesssim  \frac{1}{\delta} T \norm{\xi_1}^2 + 
         \norm{\xi_2}^2
        \maxEig{ \sum_{t=0}^{T-1}\Delta_t\Delta_t^\top}.
\end{align*}
In the last step for bounding $\sum_{t=0}^{T-1}z_{t}z_t^\top$, most of the steps are the same as \cref{eq: matrix upper bound}. The only differences are:
\begin{enumerate}
    \item Replacing the $x_t$ with $z_t$.
    \item Replacing the dimension $d$ of $x_t$ with the dimension $n+d$ of $z_t$.
\end{enumerate}
We also need two other properties
\begin{enumerate}
    \item $\E \norm{z_t}^2 \lesssim \log^2(t)$, which is proved by Eq.~(104) in \citet{wang2020exact}.
    \item 
    For $t \gtrsim \log^2(1/\delta)$ , we can use Lemma \ref{lemma: xt norm bound}'s conclusion $\EE{\norm{x_{t}}^21_{E_{\delta}}} \lesssim 1$ to prove $\EE{\norm{z_{t}}^21_{E_{\delta}}} \lesssim 1$.
    Recall that $u_t = \Kh_t x_t + \eta_t$.
        \begin{align*}
            &\norm{z_t}^2 \\
            & \quad = \norm{x_t}^2 + \norm{u_t}^2  \\
            & \quad \le \norm{x_t}^2 + 2\norm{\Kh_t x_t}^2 + 2\norm{\eta_t}^2  \\
            & \quad \le (1+2C_K^2)\norm{x_t}^2 + 2\norm{\eta_t}^2.
        \end{align*}
    Thus,
        \begin{align*}
            &\EE{\norm{z_{t}}^21_{E_{\delta}}} \\
            & \quad \le (1+2C_K^2)\EE{\norm{x_t}^21_{E_{\delta}}} + 2\EE{\norm{\eta_t}^2}  \\
            & \quad \lesssim 1.
        \end{align*}
\end{enumerate}
    \item 
\begin{align*}
    & (\xi_1 + \xi_2)^\top\left(
    \frac1\delta 
    \begin{bmatrix}
    I_n \\
    K
    \end{bmatrix} T 
    \begin{bmatrix}
    I_n \\
    K
    \end{bmatrix}^\top + 
    \begin{bmatrix}
    -K^\top \\
    I_d
    \end{bmatrix} \maxEig{ \sum_{t=0}^{T-1}\Delta_t\Delta_t^\top}
    \begin{bmatrix}
    -K^\top \\
    I_d
    \end{bmatrix}^\top
    \right)
    (\xi_1 + \xi_2) \\
    & \quad = 
    \frac1\delta 
    \xi_1^\top
    \begin{bmatrix}
    I_n \\
    K
    \end{bmatrix} T 
    \begin{bmatrix}
    I_n \\
    K
    \end{bmatrix}^\top
    \xi_1 + 
    \xi_2^\top
    \begin{bmatrix}
    -K^\top \\
    I_d
    \end{bmatrix} \maxEig{ \sum_{t=0}^{T-1}\Delta_t\Delta_t^\top}
    \begin{bmatrix}
    -K^\top \\
    I_d
    \end{bmatrix}^\top
    \xi_2 \\
    & \quad =
    \frac1\delta 
    \alpha_1^\top
    (I_n + K^\top K) T  (I_n + K^\top K)
    \alpha_1 + 
    \alpha_2^\top
    (I_d + K K^\top)
    \maxEig{ \sum_{t=0}^{T-1}\Delta_t\Delta_t^\top}
    (I_d + K K^\top)
    \alpha_2 \\
    & \quad \text{(because } I_n + K^\top K \succeq I_n)\\
    & \quad \ge
    \frac1\delta 
    \alpha_1^\top
     T  
    \alpha_1 + 
    \alpha_2^\top
    \maxEig{ \sum_{t=0}^{T-1}\Delta_t\Delta_t^\top}
    \alpha_2 \\
    & \quad =
    \frac1\delta 
    \norm{\alpha_1}^2
     T   + 
    \norm{\alpha_2}^2
    \maxEig{ \sum_{t=0}^{T-1}\Delta_t\Delta_t^\top} \\
    & \quad \gtrsim
    \frac1\delta 
    \norm{\alpha_1}^2 \norm{ \begin{bmatrix}
    I_n \\
    K
    \end{bmatrix}}^2
     T   + 
    \norm{\alpha_2}^2 \norm{ \begin{bmatrix}
    -K^\top \\
    I_d
    \end{bmatrix}}^2
    \maxEig{ \sum_{t=0}^{T-1}\Delta_t\Delta_t^\top} \\
    & \quad \gtrsim
    \frac1\delta 
    T \norm{\xi_1}^2 + 
     \norm{\xi_2}^2
    \maxEig{ \sum_{t=0}^{T-1}\Delta_t\Delta_t^\top}. 
\end{align*}
\end{enumerate}

We complete the proof by combining these two inequalities which have identical right hand side.
\end{proof}

\subsection{Proof of Lemma \ref{lemma: favorable event}}
\label{proof of lemma: favorable event}
\begin{lemma*}
\cref{alg:myAlg} applied to a system described by \cref{eq:system eq} under Assumption \ref{asm:InitialStableCondition} satisfies, for fixed $\epsilon_0 \lesssim 1$ and any  $\delta > 0$,
\begin{equation}
\label{eq: favorable event*}
    E_\delta := \mathset{\lnorm{\hat\Theta_T - \Theta}, \lnorm{\Kh_T - K} \le \epsilon_0
\text{, for all } T \gtrsim  \log^2(1/\delta)}, \PP{E_\delta} \ge 1-\delta.
\end{equation}
\end{lemma*}

\begin{proof}
By Lemma \ref{lemma: delta theta bound}, replacing $\delta$ by $\delta/T^2$, the condition on $T$ becomes $T \gtrsim \log(T^2/\delta)$, which is $T \gtrsim \log(1/\delta)$, we have
\begin{align*}
\P\left[\lnorm{\hat\Theta_T - \Theta} \gtrsim T^{-1/4}\sqrt{
\log T + \log\left(T^2/\delta\right)} 
\right] \le  \delta/T^2.
\end{align*}
Since $\sum_{T=2}^\infty 1/T^2 < \infty$, we can sum $T \gtrsim \log(1/\delta) $ these equations up:
\begin{equation*}
    \P\left[\text{Exists } T \gtrsim \log(1/\delta), \lnorm{\hat\Theta_T - \Theta} \gtrsim
    T^{-1/4}\sqrt{
3\log t + \log\left(1/\delta\right)} 
\right] \le  \delta \sum_{T=2}^\infty 1/T^2.
\end{equation*}
Let new $\delta = 3 \delta \sum_{T=2}^\infty 1/T^2$
, and we can hide the constants. As a result, we still have
\begin{equation*}
    \P\left[\text{Exists } T \gtrsim \log(1/\delta), \lnorm{\hat\Theta_T - \Theta} \gtrsim
    T^{-1/4}\sqrt{
\log T + \log\left(1/\delta\right)} 
\right] \le  \delta.
\end{equation*}

We need a uniform upper bound $\epsilon_0$ on $\lnorm{\hat\Theta_T - \Theta}$. Take $\epsilon_0$ that satisfies: $\norm{B}\epsilon_0 < 1 - \frac{1 + \rho(A+BK)}{2}$. Here $\rho(\cdot)$ is the spectural radius function. This choice of $\epsilon_0$ is to make sure even after perturbation, the system controlled by $\Kh$ is still ``stable''. Here we use quotes on stable as stability is not satisfied by the perturbation bound itself because $\rho(A+B\Kh_T) \le \rho(A+BK) + \rho(B(\Kh_T-K))$ does not hold, but this condition serves similar utility as stability as we will see in \cref{subsection: proof of lemma: xt norm bound}. 
We need to find the condition for $T$ to satisfy:
\begin{equation*}
    T^{-1/4}\sqrt{
\log t + \log\left(1/\delta\right)} \lesssim \epsilon_0.
\end{equation*}
\begin{equation*}
    T^{-1/2}(
\log t + \log\left(1/\delta\right)) \lesssim \epsilon_0^2.
\end{equation*}
\begin{equation*}
    \epsilon_0^{-4}
    \left(
\log T + \log\left(1/\delta\right)
\right)^2 \lesssim T.
\end{equation*}
Because $T$ dominates $\log(T)$, and we can hide constant $\epsilon_0^4$, the final equation can be simplified to 
\begin{equation*}
T \gtrsim \log^2(1/\delta).
\end{equation*}
Now we replace $T^{-1/4}\sqrt{
    \log T + \log\left(1/\delta\right)} $ with $\epsilon_0$, and take the complement of the whole event:
\begin{equation*}
    \P\left[\text{For all } T \gtrsim \log^2(1/\delta), \lnorm{\hat\Theta_T - \Theta} \lesssim \epsilon_0
    \right] \le  1-\delta.
\end{equation*}
By \cref{eq: delta K controlled by delta theta} we can also control $\lnorm{\Kh_T - K}$ along with $\lnorm{\hat\Theta_T - \Theta}$. 
With probability $1-\delta$, we have the following event holds:
\begin{equation*}
    E_\delta := \mathset{\lnorm{\hat\Theta_T - \Theta}, \lnorm{\Kh_T - K} \le \epsilon_0
\text{, for all } T \gtrsim  \log^2(1/\delta)}.
\end{equation*}

\end{proof}

\subsection{Proof of Lemma \ref{lemma: xt norm bound}}
\label{subsection: proof of lemma: xt norm bound}

\begin{lemma*}
\cref{alg:myAlg} applied to a system described by \cref{eq:system eq} under Assumption \ref{asm:InitialStableCondition} satisfies, for any $0 < \delta < 1/2$, $k \in \mathbb{N}$ and $T \gtrsim \log^2(1/\delta)$,
\begin{equation*}
    \EE{\norm{x_{t}}^k1_{E_{\delta}}} \lesssim 1.
\end{equation*}
\end{lemma*}

\begin{proof}
We know from the proof of Lemma 19 from \citet{wang2020exact} that for any $m > 0$:
\begin{equation}
\label{eq: x_t+m m step decomposition}
    x_{t+m} = \sum_{p=t}^{t+m-1} 
    (A + B\Kh_{t+m-1})\cdots (A + B\Kh_{p+1})(B \eta_p + \varepsilon_p) +
    (A + B\Kh_{t+m-1})\cdots (A + B\Kh_{t})x_t.
\end{equation}

By Lemma 43 from \citet{wang2020exact}, as long as $\Kh_t$ is consistent, the norm of such product $(A + B\Kh_{t+m-1})\cdots (A + B\Kh_{p+1})$ is decaying exponentially fast. More specifically, denote $L := A+BK$, which by Assumption \ref{asm:InitialStableCondition} satisfies $\rho(L) < 1$. Further define
\begin{equation*}
    \tau(L, \rho) := \sup\mathset{\norm{L^k}\rho^{-k}: k \ge 0}.
\end{equation*}
The proof of Lemma 43 of \citet{wang2020exact} showed that
\begin{align*}
    &\lnorm{(A + B\Kh_{t+m-1})\cdots (A + B\Kh_{p+1})} \\
    & \quad \le
    \tau\left(L, \frac{1 + \rho(L)}{2}\right) 
    \left(
    \frac{1+\rho(L)}{2} + \norm{B(\Kh_{t+m-1} - K)} 
    \right) \cdots 
    \left(
    \frac{1+\rho(L)}{2} + \norm{B(\Kh_{p} - K)} 
    \right).
\end{align*}
By \cref{eq: favorable event}, under event $E_\delta$, when $t \gtrsim \log^2(1/\delta)$, the difference $\Kh_{t} - K$ is uniformly bounded by $\epsilon_0$. Denote $\rho_0 = \frac{1 + \rho(L)}{2} + \norm{B}\epsilon_0 < 1$. When $p \gtrsim  \log^2(1/\delta)$, 
\begin{equation*}
    \lnorm{(A + B\Kh_{t+m-1})\cdots (A + B\Kh_{p+1})1_{E_{\delta}}} 
    \lesssim \rho_0^{t+m-p}.
\end{equation*}
Apply this equation to \cref{eq: x_t+m m step decomposition}:
\begin{align*}
    &\EE{\norm{x_{t+m}}^k 1_{E_{\delta} }}  \\
    & \quad \lesssim  \EE{\left(
    \norm{\sum_{p=t}^{t+m-1} 
    \rho_0^{t+m-p} (B \eta_p + \varepsilon_p)}
    +
    \rho_0^{m}\norm{x_t}\right)^k 1_{E_{\delta}}} \\
    & \quad \text{(By Holder's inequality)} \\
    &  \quad \lesssim  2^{k-1} \EE{
        \norm{\sum_{p=t}^{t+m-1} 
    \rho_0^{t+m-p} (B \eta_p + \varepsilon_p)}^k 1_{E_{\delta}}
    +
    \rho_0^{km}\norm{x_t}^k 1_{E_{\delta}}} \\
    & \quad \lesssim
    \EE{
        \norm{\sum_{p=t}^{t+m-1} 
    \rho_0^{t+m-p} (B \eta_p + \varepsilon_p)}^k} + 
    \rho_0^{km} \EE{\norm{x_{t}}^k1_{E_{\delta}}}.
\end{align*}
Consider the variance of $\sum_{p=t}^{t+m-1} 
    \rho_0^{t+m-p} (B \eta_p + \varepsilon_p)$:
\begin{align*}
    &\Var{\sum_{p=t}^{t+m-1} 
    \rho_0^{t+m-p} (B \eta_p + \varepsilon_p)} \\
    & \quad =  \sum_{p=t}^{t+m-1}\rho_0^{2(t+m-p)} \Var{
    (B \eta_p + \varepsilon_p)} \\
    & \quad \lesssim  \sum_{p=t}^{t+m-1}\rho_0^{2(t+m-p)} =  \sum_{i=1}^{m}\rho_0^{2i} \lesssim 1.
\end{align*}
Since $\sum_{p=t}^{t+m-1} 
    \rho_0^{t+m-p} (B \eta_p + \varepsilon_p)$ is Gaussian with finite variance,
     the first item is of constant order for any $m$. Thus 
\begin{equation*}
    \EE{\norm{x_{t+m}}^k 1_{E_{\delta} }} \lesssim 1 + \rho_0^{km} \EE{\norm{x_{t}}^k1_{E_{\delta}}}.
\end{equation*}
Replace $t \leftarrow m$, and $m \leftarrow t-m$, then for $m \gtrsim \log^2(1/\delta)$,
\begin{equation*}
    \EE{\norm{x_{t}}^k 1_{E_{\delta} }} \lesssim 1 + \rho_0^{k(t-m)} \EE{\norm{x_{m}}^k1_{E_{\delta}}}.
\end{equation*}
Since the $\Kh_t$ in \cref{alg:myAlg} cannot have norm greater than $C_K$, we have
    \begin{align*}
    &\EE{\norm{x_{m}}^k1_{E_{\delta}}} \\
        & \quad \le \E\norm{x_m}^k \\
        & \quad \le \E\left((\norm{A}+\norm{B}\norm{\Kh_m})\norm{x_{m-1}} + \norm{B}\norm{\eta_m} + \norm{\varepsilon_m}\right)^k \\
        & \quad \text{(By Holder's inequality)} \\
        & \quad \le 3^{k-1}\left((\norm{A}+\norm{B}C_K)^k\E\norm{x_{m-1}}^k + \norm{B}^k\E\norm{\eta_m}^k + \norm{\varepsilon_m}^k\right)\\
        & \quad \le 3^{k-1}\left((\norm{A}+\norm{B}C_K)^k\E\norm{x_{m-1}}^k + \norm{B}^k\sigma_{\eta}^k + \sigma_\varepsilon^k\right).
    \end{align*}
By iterating this inequality down to $\norm{x_0}^2$, we know that $\E\norm{x_m}^k  \lesssim C^{km}$ for some constant $C$. Thus, we know
\begin{equation*}
    \EE{\norm{x_{t}}^k 1_{E_{\delta} }} \lesssim 1 + \rho_0^{k(t-m)} C^{km}.
\end{equation*}
Since $\rho_0 < 1$, we can take $t \ge (\log_{1/\rho_0}(C) + 1)m$ which satisfies $\rho_0^{k(t-m)} C^{km} \le 1$. Because we require $m \gtrsim \log^2(1/\delta)$, the condition for $t \ge (\log_{1/\rho_0}(C) + 1)m$ is still $t \gtrsim \log^2(1/\delta)$, which satisfies

\begin{equation*}
    \EE{\norm{x_{t}}^k 1_{E_{\delta}} } \lesssim 1.
\end{equation*}

\end{proof}

\section{Proof of \cref{theorem:regret}}
\label{subsection: proof of theorem:regret}
\begin{theorem*}
\cref{alg:myAlg} applied to a system described by \cref{eq:system eq} under Assumption \ref{asm:InitialStableCondition} satisfies
\begin{equation}
    \label{eq: tight regret*}
    \mathcal{R}(U,T) = \bigOp{\sqrt{T}}.
\end{equation}
\end{theorem*}

\begin{proof}
Recall from Lemma \ref{lemma: cost for our algo and optimal algo} that 
\begin{equation*}
        \mathcal{J}(U,T)  = \sum_{t=1}^{T} \tilde\varepsilon_{t}^\top P \tilde\varepsilon_{t} +
    \sum_{t=1}^{T} \eta_{t}^\top R \eta_{t} + \bigOp{T^{1/2}},
\end{equation*}
and 
\begin{equation*}
        \mathcal{J}(U^*,T)
        = \sum_{t=1}^{T} \varepsilon_{t}^\top P \varepsilon_{t}  + \bigOp{T^{1/2}}.
\end{equation*}
Thus
\begin{align}
\label{eq: diff myalg optalg}
\begin{split}
    &\mathcal{R}(U,T) \\
    &\quad = \mathcal{J}(U,T)- \mathcal{J}(U^*,T) \\
    &\quad = \sum_{t=1}^{T} \tilde\varepsilon_{t}^\top P \tilde\varepsilon_{t} +
    \sum_{t=1}^{T} \eta_{t}^\top R \eta_{t} - \sum_{t=1}^{T} \varepsilon_{t}^\top P \varepsilon_{t} + \bigOp{T^{1/2}} \\
    &\quad = 2\sum_{t=1}^{T} \varepsilon_{t}^\top P (B\eta_t) + \sum_{t=1}^{T} (B\eta_t)^\top P (B\eta_t) + 
    \sum_{t=1}^{T} \eta_{t}^\top R \eta_{t} + \bigOp{T^{1/2}}.
\end{split}
\end{align}
Recall that $\eta_t \sim \calN(0, \sigma_\eta^2 t^{-1/2}I_\inputdim))$.
\begin{equation*}
    \begin{split}
        \E \sum_{t=1}^{T} \eta_t^\top R \eta_t  = \sum_{t=1}^{T} \Tr (R \E \eta_t \eta_t^\top ) = \sum_{t=1}^{T} \Tr (R \sigma_\eta^2 t^{-1/2}  ) = \bigO{T^{1/2}}
    .\end{split}
\end{equation*}

\begin{equation*}
    \begin{split}
        \Var{\sum_{t=1}^{T} \eta_t^\top R \eta_t}  = \sum_{t=1}^{T} \Var{\eta_t^\top R \eta_t} = \sum_{t=1}^{T} \bigO{t^{-1}} = \bigO{\log(T)}
    .\end{split}
\end{equation*}
The standard error is of smaller order than the expectation. Thus, $\sum_{t=1}^{T} \eta_t^\top R \eta_t  = \bigOp{T^{1/2}}$. Similarly, $\sum_{t=1}^{T} (B\eta_t)^\top P (B\eta_t)  = \bigOp{T^{1/2}}$.

It remains to consider the order of $\sum_{t=1}^{T} \varepsilon_{t}^\top P (B\eta_t)$. Its expectation is $0$.
\begin{equation*}
    \begin{split}
        \E \sum_{t=1}^{T} \varepsilon_{t}^\top P (B\eta_t) = 0
    .\end{split}
\end{equation*}
The variance is
\begin{equation*}
    \begin{split}
        \Var{\sum_{t=1}^{T} \varepsilon_{t}^\top P (B\eta_t)} 
        = \sum_{t=1}^{T}\Var{ \varepsilon_{t}^\top P (B\eta_t)}  
         = \sum_{t=1}^{T} \bigO{t^{-1/2}}
         = \bigO{T^{1/2}}
    .\end{split}
\end{equation*}
The standard error is of order $T^{1/4}$. Thus, $\sum_{t=1}^{T} \varepsilon_{t}^\top P (B\eta_t) = \bigOp{T^{1/4}} = o_p(T^{1/2})$. Using these results, \cref{eq: diff myalg optalg} becomes:
\begin{align*}
    &\mathcal{R}(U,T) \\
    &\quad = 2\sum_{t=1}^{T} \varepsilon_{t}^\top P (B\eta_t) + \sum_{t=1}^{T} (B\eta_t)^\top P (B\eta_t) + 
    \sum_{t=1}^{T} \eta_{t}^\top R \eta_{t} + \bigOp{T^{1/2}} \\
    &\quad = \bigOp{T^{1/2}}.
\end{align*}
\end{proof}

\subsection{Proof of Lemma \ref{lemma: cost for our algo and optimal algo}}
\label{subsection: proof of lemma: cost for our algo and optimal algo}
\begin{lemma*}
\cref{alg:myAlg} applied to a system described by \cref{eq:system eq} under Assumption \ref{asm:InitialStableCondition} satisfies, 
\begin{equation*}
        \mathcal{J}(U,T)  = \sum_{t=1}^{T} \tilde\varepsilon_{t}^\top P \tilde\varepsilon_{t} +
    \sum_{t=1}^{T} \eta_{t}^\top R \eta_{t} + \bigOp{T^{1/2}}.
\end{equation*}
and 
\begin{equation*}
        \mathcal{J}(U^*,T)
        = \sum_{t=1}^{T} \varepsilon_{t}^\top P \varepsilon_{t}  + \bigOp{T^{1/2}},
\end{equation*}
where $\varepsilon_t$ is the system noise and $\eta_t$ is the exploration noise in \cref{alg:myAlg}, and $\tilde{\varepsilon}_t = B\eta_t + \varepsilon_t$. 
\end{lemma*}

 We only prove the first equation because the second equation is a simplified version of the first equation (with $\eta_t = 0$ and $\Kh_t = K$). 

Recursively applying system equations $x_{t+1} = Ax_t + Bu_t + \varepsilon_t$ and $u_t = \Kh_t x_t + \eta_t$ we have:
\begin{equation}
    \label{eq: StateExpansion}
    x_{t} = \sum_{p=0}^{t-1}(A+B \Kh_{t-1})\cdots(A+B \Kh_{p+1})(B \eta_p+\varepsilon_p) +(A+B \Kh_{t-1})\cdots(A+B K_0)x_0.
\end{equation}

Notice that the state $x_t$ has the same expression as if the system had noise $\tilde{\varepsilon}_t = B\eta_t + \varepsilon_t$ and controller $\tilde{u}_t = \Kh_t x_t$. We wish to switch to the new system because there are some existing tools with controls in the form of $\tilde{u}_t = \Kh_t x_t$.

We are interested in the cost 

\[\mathcal{J}(U,T) = \sum_{t=1}^{T} x_t^\top Qx_t + u_t^\top Ru_t \quad \text{with $u_t = \Kh_t x_t + \eta_t$}.\]
%
We will first show in \cref{subsection: Cost of new system} the new system cost is
\begin{lemma}
\label{lemma: new sys cost}
\cref{alg:myAlg} applied to a system described by \cref{eq:system eq} under Assumption \ref{asm:InitialStableCondition} satisfies, 
\begin{equation*}
    \sum_{t=1}^{T} x_t^\top Qx_t + \tilde{u}_t^\top R\tilde{u}_t = \sum_{t=1}^{T} \tilde\varepsilon_{t}^\top P \tilde\varepsilon_{t} +
    \bigOp{T^{1/2}}
.\end{equation*}
\end{lemma}
and then prove in \cref{subsection: Cost difference induced by transformation} that the difference between the original cost and new cost is 
\begin{lemma}
\label{lemma: trans cost}
\cref{alg:myAlg} applied to a system described by \cref{eq:system eq} under Assumption \ref{asm:InitialStableCondition} satisfies, 
\begin{equation*}
    \sum_{t=1}^{T} u_t^\top  Ru_t -  \tilde{u}_t^\top R\tilde{u}_t  =
    \sum_{t=1}^T \eta_t^\top R \eta_t + 
    o\left(T^{1/4}\log^{\frac{3}{2}}(T) \right) \as
\end{equation*}
\end{lemma}
%
Combining the above two equations, we conclude that
\begin{align*}
    \mathcal{J}(U,T)
    &= 
    \left[\sum_{t=1}^{T} x_t^\top Qx_t + u_t^\top Ru_t\right] \\
    &= \sum_{t=1}^{T} \tilde\varepsilon_{t}^\top P \tilde\varepsilon_{t} +
    \sum_{t=1}^{T} \eta_{t}^\top R \eta_{t} + \bigOp{T^{1/2}}
.\end{align*}
The optimal controller $U^*$ is a simplified version of $U$ from \cref{alg:myAlg} with $\eta_t = 0$ and $\Kh_t - K = 0$.  With the same proof we can show that
\begin{align*}
    \mathcal{J}(U^*,T)
    &= \sum_{t=1}^{T} \varepsilon_{t}^\top P \varepsilon_{t}  + \bigOp{T^{1/2}}.
    \end{align*}

\subsubsection{Cost of new system}
\label{subsection: Cost of new system}
\begin{lemma*}
\cref{alg:myAlg} applied to a system described by \cref{eq:system eq} under Assumption \ref{asm:InitialStableCondition} satisfies, 
\begin{equation*}
    \sum_{t=1}^{T} x_t^\top Qx_t + \tilde{u}_t^\top R\tilde{u}_t = \sum_{t=1}^{T} \tilde\varepsilon_{t}^\top P \tilde\varepsilon_{t} +
    \bigOp{T^{1/2}}
.\end{equation*}
\end{lemma*}

\begin{proof}
Next we proceed as if our system was  $x_t$ with system noise $\tilde\varepsilon_t = B \eta_t+\varepsilon_t$ and controller $\tilde{u}_t = \Kh_t x_t$. 
The key idea of the following proof is from Appendix C of \citet{fazel2018global}.


We are interested in the cost

\[ \sum_{t=1}^{T} x_t^\top Qx_t + \tilde{u}_t^\top R\tilde{u}_t  \quad \text{with $\tilde{u}_t = \Kh_t x_t $},\]

which can be written as 
\begin{align}
\label{eq: regret decomposition}
    \sum_{t=1}^{T} x_t^\top Qx_t + \tilde{u}_t^\top R\tilde{u}_t 
    =& \sum_{t=1}^{T} x_t^\top Qx_t + (\Kh_t x_t)^\top R\Kh_t x_t \nonumber\\
    =& \sum_{t=1}^{T} x_t^\top (Q + \Kh_t^\top R \Kh_t)x_t  \nonumber\\
    =& \sum_{t=1}^{T} \left[x_t^\top (Q + \Kh_t^\top R \Kh_t)x_t + x_{t+1}^\top P x_{t+1} - x_{t}^\top P x_{t}
    \right] 
    + x_1^\top Px_1 - x_{T+1}^\top P x_{T+1} \nonumber\\
    =& \sum_{t=1}^{T} \left[x_t^\top (Q + \Kh_t^\top R \Kh_t)x_t + ((A+B\Kh_t)x_t +\tilde\varepsilon_{t})^\top P ((A+B\Kh_t)x_t +\tilde\varepsilon_{t}) - x_{t}^\top P x_{t}\right] \nonumber\\
    & \quad + x_1^\top Px_1 - x_{T+1}^\top P x_{T+1} \nonumber\\
    &\quad \text{(by Lemma 18 in \citet{wang2020exact}) } \nonumber\\
    =& \sum_{t=1}^{T} \Big[ x_t^\top (Q + \Kh_t^\top R \Kh_t)x_t + x_t^\top (A+B\Kh_t)^\top P (A+B\Kh_t)x_t  - x_{t}^\top P x_{t} \nonumber\\
    & \quad+ 2  \tilde\varepsilon_{t}^\top P (A+B\Kh_t)x_t + \tilde\varepsilon_{t}^\top P \tilde\varepsilon_{t}
    \Big]
    + \logOO_p(1). 
\end{align}
We constructed the specific form of the first term on purpose. The following lemma translates the first term into a quadratic term with respect to $\Kh_t - K$. We use the Lemma 25 from \citet{wang2020exact}:
\begin{lemma}[Lemma 25 from \citet{wang2020exact}]
\label{lem: useful lemma from fazel}
For any $\Kh$ with suitable dimension,
\begin{equation*}
\begin{split}
    &x^\top (Q + \Kh^\top R \Kh)x + x^\top (A+B\Kh)^\top P (A+B\Kh)x - x^\top P x \\
    & \qquad = x^\top (\Kh-K)^\top( R + B^\top P B) (\Kh-K)x
.\end{split}
\end{equation*}
\end{lemma}
As a result
\begin{equation*}
\begin{split}
    \sum_{t=1}^{T} x_t^\top Qx_t + \tilde{u}_t^\top R\tilde{u}_t 
    =& \sum_{t=1}^{T} \bigg[x_t^\top (\Kh_t-K)^\top( R + B^\top P B) (\Kh_t-K)x_t \\
    &\quad + 2  \tilde\varepsilon_{t}^\top P (A+B\Kh_t)x_t + \tilde\varepsilon_{t}^\top P \tilde\varepsilon_{t}\bigg]
    + \logOO_p(1)
.\end{split}
\end{equation*}

Now we have three terms, and we will estimate the order of each of these three terms.

\begin{enumerate}
    \item The first term we consider is $\sum_{t=1}^{T} x_t^\top (\Kh_t-K)^\top( R + B^\top P B) (\Kh_t-K)x_t$.
    For any $0 < \delta < 1/2$, 
    \begin{align*}
        &\EE{\sum_{t=1}^{T} x_t^\top (\Kh_t-K)^\top( R + B^\top P B) (\Kh_t-K)x_t 1_{E_{\delta}}} \\
        & \quad \le  
        \EE{\sum_{t=1}^{T} \norm{x_t}^2 \norm{\Kh_t-K}^2 \norm{ R + B^\top P B} 1_{E_{\delta}}} \\
        & \quad \lesssim
        \EE{\sum_{t=1}^{T} \norm{x_t}^2 \norm{\Kh_t-K}^21_{E_{\delta}} } \\
        & \quad \lesssim  
        \EE{\sum_{t=1}^{T} t^{1/2} \norm{\Kh_t-K}^4 + t^{-1/2} \norm{x_t}^4 1_{E_{\delta}}}\\
        & \quad \text{(By \cref{eq:E K estimation err ^4 bound} and the same inequalities as in \cref{eq: deltat eigenvalue bound})}\\
        & \quad \lesssim 
        \sum_{t=T_0}^{T} t^{-1/2} + \sum_{t=1}^{T_0-1} t^{1/2} (C_K + \norm{K})^4 + 
        \log^2(1/\delta)+ T^{1/2} 
        \\
        & \quad \lesssim 
        T^{1/2} + \log^2(1/\delta).
    \end{align*}
    Then for any $0 < \delta < 1/2$,  we have
    \begin{equation*}
        \PP{\sum_{t=1}^{T} x_t^\top (\Kh_t-K)^\top( R + B^\top P B) (\Kh_t-K)x_t 1_{E_{\delta}} \gtrsim \frac{1}{\delta} (T^{1/2} + \log^2(1/\delta))} \le \delta.
    \end{equation*}
    \begin{equation*}
        \PP{\sum_{t=1}^{T} x_t^\top (\Kh_t-K)^\top( R + B^\top P B) (\Kh_t-K)x_t  \gtrsim \frac{1}{\delta}(\log^2(1/\delta) + 1)T^{1/2}} \le 2\delta.
    \end{equation*}
    By big O in probability notation, this implies
    \begin{equation*}
        \sum_{t=1}^{T} x_t^\top (\Kh_t-K)^\top( R + B^\top P B) (\Kh_t-K)x_t  = \bigOp{T^{1/2}}.
    \end{equation*}

    \item The second term we consider is $\sum_{t=1}^{T} \tilde\varepsilon_{t}^\top P (A+B\Kh_t)x_t$. Notice that $\tilde\varepsilon_{t} = \varepsilon_{t} + B\eta_t \independent (A+B\Kh_t)x_t$. Then 
    \begin{equation*}
        \E \sum_{t=1}^{T} \tilde\varepsilon_{t}^\top P (A+B\Kh_t)x_t = 0
    .\end{equation*}
    \begin{align*}
        &\E (\sum_{t=1}^{T} \tilde\varepsilon_{t}^\top P (A+B\Kh_t)x_t 1_{E_\delta})^2  \\
        &\quad =\sum_{t=1}^{T} \E (\tilde\varepsilon_{t}^\top P (A+B\Kh_t)x_t)^21_{E_\delta} \\
        &\quad \le  \sum_{t=1}^{T} \EE{ \norm{\tilde\varepsilon_{t}}^2 \norm{P}^2 \norm{(A+B\Kh_t)}^2 \norm{x_t}^2 1_{E_\delta}}\\
        &\quad \text{ ($\norm{\Kh_t} \le C_K$ based on \cref{alg:myAlg} design)}  \\
        &\quad \le \sum_{t=1}^{T}  \norm{P}^2 (\norm{A} + \norm{B} C_K)^2 \E \norm{\tilde\varepsilon_{t}}^2 \EE{ \norm{x_t}^21_{E_\delta}} \\
        &\quad \lesssim \sum_{t=1}^{T}\EE{ \norm{x_t}^21_{E_\delta}} \\
        & \quad \text{(By the inequalities in \cref{eq: deltat eigenvalue bound})} \\
        &\quad \lesssim T + \log^2(1/\delta) \log\log(1/\delta) \\
        &\quad \lesssim T + \log^3(1/\delta)
    .
    \end{align*}
    Then for any $0 < \delta < 1/2$,  we have
    \begin{equation*}
        \PP{\sum_{t=1}^{T} \tilde\varepsilon_{t}^\top P (A+B\Kh_t)x_t 1_{E_\delta} \gtrsim \sqrt{\frac{1}{\delta} (T + \log^3(1/\delta)) }} \le \delta
    \end{equation*}
    and
    \begin{equation*}
        \PP{\sum_{t=1}^{T} \tilde\varepsilon_{t}^\top P (A+B\Kh_t)x_t  \gtrsim \sqrt{\frac{1}{\delta} (1 + \log^3(1/\delta)) T }} \le 2\delta.
    \end{equation*}
    By big O in probability notation, this implies
    \begin{equation}
    \label{eq: regret 2nd item}
        \sum_{t=1}^{T} \tilde\varepsilon_{t}^\top P (A+B\Kh_t)x_t = \bigOp{T^{1/2}}
    .\end{equation}
    
    \item The third term is $\sum_{t=1}^{T} \tilde\varepsilon_{t}^\top P \tilde\varepsilon_{t}$ and we leave that in the equation.
    
\end{enumerate}
Summing up the three parts we have:
\begin{equation*}
    \sum_{t=1}^{T} x_t^\top Qx_t + \tilde{u}_t^\top R\tilde{u}_t = \sum_{t=1}^{T} \tilde\varepsilon_{t}^\top P \tilde\varepsilon_{t} +
    \bigOp{T^{1/2}}
.\end{equation*}
\end{proof}

\subsubsection{Cost difference induced by transformation}
\label{subsection: Cost difference induced by transformation}
\begin{lemma*}
\cref{alg:myAlg} applied to a system described by \cref{eq:system eq} under Assumption \ref{asm:InitialStableCondition} satisfies, 
\begin{equation*}
    \sum_{t=1}^{T} u_t^\top  Ru_t -  \tilde{u}_t^\top R\tilde{u}_t  =
    \sum_{t=1}^T \eta_t^\top R \eta_t + 
    o\left(T^{1/4}\log^{\frac{3}{2}}(T) \right) \as
\end{equation*}
\end{lemma*}

\begin{proof}
The difference is expressed as
\begin{equation*}
\begin{split}
        \sum_{t=1}^{T} u_t^\top  Ru_t -  \tilde{u}_t^\top R\tilde{u}_t     
    =&  \sum_{t=1}^{T}(\Kh_t x_t + \eta_t)^\top  R (\Kh_t x_t + \eta_t)   - \sum_{t=1}^{T} (\Kh_t x_t)^\top  R (\Kh_t x_t) \\
    =& 2\sum_{t=1}^{T} (\Kh_t x_t)^\top  R \eta_t +  \sum_{t=1}^{T} \eta_t^\top R \eta_t 
.\end{split}
\end{equation*}
Eq.~(83) of \citet{wang2020exact} shows that
\begin{equation*}
    \sum_{t=1}^{T} (\Kh_t x_t)^\top  R \eta_t = o\left(T^{1/4}\log^{\frac{3}{2}}(T) \right) \as
\end{equation*}

As a conclusion, 
\begin{equation*}
    \sum_{t=1}^{T} u_t^\top  Ru_t -  \tilde{u}_t^\top R\tilde{u}_t  =
    \sum_{t=1}^T \eta_t^\top R \eta_t + 
    o\left(T^{1/4}\log^{\frac{3}{2}}(T) \right) \as
\end{equation*}
\end{proof}

\end{document}